\RequirePackage{fix-cm}
\documentclass[twocolumn]{svjour3}          
\smartqed  
\usepackage{graphicx}

\usepackage{url}
\usepackage{diagbox}
\usepackage{makecell}
\usepackage{amsmath,amssymb,amsfonts}
\usepackage{graphicx}
\usepackage{hyperref}
\hypersetup{colorlinks,linkcolor=blue,citecolor=blue,urlcolor=magenta, linktocpage,plainpages=false, bookmarksopen = true}
\usepackage{multirow}
\usepackage{arydshln}
\usepackage[flushleft]{threeparttable}
\usepackage{soul}

\usepackage[sort,numbers]{natbib}
\usepackage{romannum}
\newcommand{\smallbf}[1]{{\fontsize{7}{12}\selectfont{\textbf{{#1}}}}}

\journalname{}
\begin{document}
\pagenumbering{arabic}
\sloppy
\title{A Novel Measure to Evaluate Generative Adversarial Networks Based on Direct Analysis of Generated Images
}


\author{Shuyue Guan         \and
        Murray Loew 
}


\institute{Shuyue Guan \at
              \textit{Related codes can be found in author’s website linked up with the ORCID: \href{https://orcid.org/0000-0002-3779-9368}{0000-0002-3779-9368}} \\
              \email{frankshuyueguan@gwu.edu}           
           \and
           Murray Loew \at
              \textit{Corresponding author} \\
              \email{loew@gwu.edu} \\
              \\
           Department of Biomedical Engineering, \\
           George Washington University, Washington DC, USA
}

\date{}

\maketitle

\begin{abstract}
The Generative Adversarial Network (GAN) is a state-of-the-art technique in the field of deep learning. A number of recent papers address the theory and applications of GANs in various fields of image processing. Fewer studies, however, have directly evaluated GAN outputs. Those that have been conducted focused on using classification performance, \textit{e.g.}, \textit{Inception Score} (IS) and statistical metrics, \textit{e.g.}, \textit{Fréchet Inception Distance} (FID). Here, we consider a fundamental way to evaluate GANs \ul{by directly analyzing the images they generate}, instead of using them as inputs to other classifiers. We characterize the performance of a GAN as an image generator according to three aspects: 1) Creativity: non-duplication of the real images. 2) Inheritance: generated images should have the same style, which retains key features of the real images. 3) Diversity: generated images are different from each other. A GAN should not generate a few different images repeatedly. Based on the three aspects of ideal GANs, we have designed the \textit{Likeness Score} (LS) to evaluate GAN performance, and have applied it to evaluate several typical GANs. We compared our proposed measure with two commonly used GAN evaluation methods: IS and FID, and four additional measures. Furthermore, we discuss how these evaluations could help us deepen our understanding of GANs and improve their performance.
\end{abstract}
\textbf{Declarations} Not applicable.

\begin{figure*}[htbp]
    \centering
    \includegraphics[width=.95\textwidth]{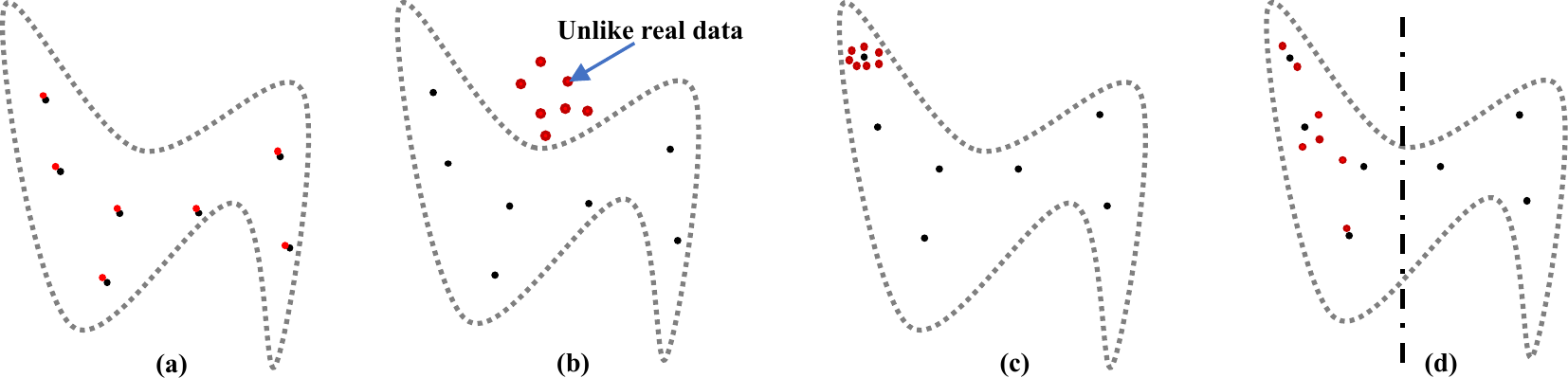}
    \caption{Problems of generated images from the perspective of distribution. The area of dotted line is the distribution of real images. The dark-blue dots are real samples and red dots are generated images. (a) is overfitting, lack of Creativity. (b) is lack of Inheritance. (c) is called mode collapse for GAN and (d) is mode dropping. Both (c) and (d) are examples of lack of Diversity.}
    \label{fig:1}
\end{figure*}

\section{Introduction}
As neural-network based generators, Generative Adversarial Networks (GANs) were introduced by~\citet{Goodfellow2014Generative} in 2014, and they have become a state-of-the-art technique in the field of deep learning~\cite{Hong2019How}. Recently, the number of types of GANs has grown to about 500~\cite{Hindupur2018the-gan-zoo:} and a substantial number of studies are about the theory and applications of GANs in various fields of image processing, including image translation~\cite{Wang2018Perceptual,Yi2017DualGAN:}, object detection~\cite{Li2017Perceptual}, super-resolution~\cite{Ledig2017Photo-Realistic}, image synthesis~\cite{Pan2019Recent} and image blending~\cite{Wu2019GP-GAN:}. Compared to the theoretical progress and applications of GANs, however, fewer studies have focused on evaluating or measuring GANs’ performance~\cite{Borji2019Pros}. Most existing GANs’ measures have been conducted using classification performance (\textit{e.g.}, \textit{Inception Score}) and statistical metrics (\textit{e.g.}, \textit{Fréchet Inception Distance}). A more fundamental alternative approach to evaluate a GAN is to directly analyze the images it generated, instead of using them as inputs to other classifiers (\textit{e.g.}, \textit{Inception} network) and then analyzing the outcomes. 

In this study, we propose a fundamental way to analyze GAN-generated images quantitatively and qualitatively. We briefly introduce the two commonly used GAN evaluation methods: \textit{Inception Score} (IS)~\cite{Salimans2016Improved} and \textit{Fréchet Inception Distance} (FID)~\cite{Heusel2017GANs}, and four additional measures: \textit{1-Nearest Neighbor classifier} (1NNC)~\cite{Lopez-Paz2017Revisiting}, \textit{Mode Score} (MS)~\cite{che2016mode}, \textit{Activation Maximization} (AM) score~\cite{zhou2018activation}, and \textit{Sliced Wasserstein distance} (SWD)~\cite{bonneel_sliced_2015}. We then compare those results with our proposed measure. In addition, we discuss how these evaluations could help us to deepen our understanding of GANs and to improve their performance.

\subsection{GAN Evaluation Metrics}
The optimal GAN for images can generate images that have the same distribution as real samples (used for training), are different from real ones (not duplication), and have variety. Expectations of generated images could be described by three aspects: 1) non-duplication of the real images, 2) generated images should have the same style, which we take to mean that their distribution is close to that of the real images, and 3) generated images are different from each other. Therefore, we evaluate the performance of a GAN as an image generator according to the \ul{three aspects}:
\begin{itemize}
    \item \textbf{Creativity}: non-duplication of the real images. It checks for overfitting by GANs.
    
    \item \textbf{Inheritance} (or visual fidelity): generated images should have the same style, which retains key features of the real (input) images. And this is traded off with the creativity property because generated images should not be too similar nor too dissimilar to the real ones.
    
    \item \textbf{Diversity}: generated images are different from each other. A GAN should not generate a few dissimilar images repeatedly.
\end{itemize}
Fig.~\ref{fig:1} displays four counterexamples of ideal generated images.

We introduce a distance-based separability index and use it to define the measure: \textit{Likeness Score} (LS) to evaluate GAN performance according to the three expectations of ideal generated images. LS offers a direct way to measure difference or similarity between images based on the Euclidean distance and has a simple and uniform framework for the three aspects of ideal GANs and depends less on visual evaluation.

The proposed LS measure is applied to analyze the generated images directly, without using pre-trained classifiers. We applied the measure to outcomes of several typical GANs: DCGAN~\cite{Radford2016Unsupervised}, WGAN-GP~\cite{Gulrajani2017Improved}, SNGAN~\cite{Miyato2018Spectral} LSGAN~\cite{Mao2017Least} and SAGAN~\cite{Zhang2019Self-Attention} on various image datasets. Results show that the LS can reflect the performance of GAN well and are very competitive with other compared measures. In addition, the LS is stable with respect to the number of images and could provide an explanation of results in terms of the three respects of ideal GANs.

\subsection{Related Works}
\label{related}
Recently, the two most widely applied indexes to evaluate GANs performance are the \textit{Inception Score} (IS)~\cite{Salimans2016Improved} and \textit{Fréchet Inception Distance} (FID)~\cite{Heusel2017GANs}. They both depend on the pre-trained Inception~network~\cite{Szegedy2016Rethinking} that was trained on the ImageNet~\cite{Deng2009ImageNet:} dataset.

\subsubsection{KL Divergence Based Evaluations}
From the perspective of the three aspects for ideal GANs, the IS focuses on measuring the inheritance and diversity. Specifically, we let $x\in G$ be a generated image; $y=\texttt{InceptionNet}(x)$ is the label obtained from the pre-trained Inception network by inputting image $x$. For all generated images, we have the label set $Y$. $H(Y)$ defines the diversity ($H(\cdot)$ is entropy) because the variability of labels reflects the variability of images. $H(Y|G)$ could show the inheritance because a good generated image can be well recognized and classified, and thus the entropy of $p(y|x)$ should be small. Therefore, an ideal GAN will maximize $H(Y)$ and minimize $H(Y|G)$. Equivalently, the goal is to maximize:
\[
H(Y)-H(Y|G)=E_G[D_{KL}\left(p(y|x)\|p(y)\right)]
\]
$D_{KL}$ is the Kullback–Leibler (KL) divergence of two distributions~\cite{Kullback1951On}. The IS index is defined:
\[
IS(G)=\exp \left(E_G[D_{KL}\left(p(y|x)\|p(y)\right)]\right)
\]

The IS mainly shows diversity and reflects inheritance to some extent; a larger value of IS indicates that a GAN's performance is better. The substantial limitations of IS are:
\begin{enumerate}
    \item \ul{It depends on classification of images by the Inception network}, which is by trained ImageNet, and employs generated data without exploiting real data. Thus, IS may not be proper to use on other images or non-classification tasks because it cannot properly show the inheritance if the data are different from those used in ImageNet.
    
    \item \ul{Creativity is not considered by the IS} because it ignores the real data. And it has no ability to detect overfitting. For example, if the set of generated images was a copy of the real images and very similar to images of ImageNet, IS will give a high score.
\end{enumerate}

The main drawback of the IS is disregard of real data. Thus, to improve the performance of IS, the \textit{Mode Score} (MS)~\cite{che2016mode} and \textit{Activation Maximization} (AM) score~\cite{zhou2018activation} include real data in their computations. Specifically, we let $z\in R$ be a real image; $y_r=\texttt{InceptionNet}(z)$ is the label obtained from the pre-trained Inception network by inputting the real image $z$. The MS is then defined as:
\begin{multline*}
MS(R, G)= \\ \exp \left(E_G[D_{KL}\left(p(y|x)\|p(y_r)\right)]-D_{KL}\left(p(y)\|p(y_r)\right)\right)    
\end{multline*}
And the AM is defined as:
\[
AM(R, G)= E_G[H\left(y|x\right)]+D_{KL}\left(p(y_r)\|p(y)\right)
\]
Like the IS, larger value of MS is better; but smaller value of AM is better.

\subsubsection{Distance-based Evaluations}
The FID also exploits real data and uses the pre-trained Inception network. Instead of output labels it uses feature vectors from the final pooling layers of the \texttt{InceptionNet}. All real and generated images are input to the network to extract their feature vectors. 

Let $\varphi(\cdot) = \texttt{InceptionNet\_lastPooling}(\cdot)$ be the feature extractor and let $F_r=\varphi(R),\ F_g=\varphi(G)$ be two groups of feature vectors extracted from real and generated image sets. Consider that the distributions of $F_r,\ F_g$ are multivariate Gaussian:
\[
F_r\sim N\left(\mu_r,\Sigma_r\right);\ F_g\sim N\left(\mu_g,\Sigma_g\right)
\]
The difference of two Gaussians is measured by the Fréchet distance:
\begin{multline*}
FID\left(R,G\right)= \\ \Arrowvert\mu_r-\mu_g\Arrowvert_2^2 + Tr\left(\Sigma_r+\Sigma_g-2\left(\Sigma_r \Sigma_g\right)^{\frac{1}{2}}\right)
\end{multline*}

In fact, FID measures the difference between distributions of real and generated images; that agrees with the goal of GAN training -- to minimize the difference between the two distributions. The FID measure, however, depends on the \ul{multivariate Gaussian distribution assumption} of $F_r$ and $F_g$: {$\textstyle F_r\sim N\left(\mu_r,\Sigma_r\right);\ F_g\sim N\left(\mu_g,\Sigma_g\right)$}. The assumption of multivariate Gaussian distributions of feature vectors cannot be always guaranteed because some features may not be Gaussian distributed. And in a high-dimensional space, because of the \textit{curse of dimensionality}, the amount of data may be not large enough to form a multivariate Gaussian distribution (because that requires a large amount of data according to the \textit{Central Limit Theorem}). In addition, as with IS, FID depends on the pre-trained Inception network.

To avoid the Gaussian assumption, we can directly compute the Wasserstein distance~\cite{ruschendorf_wasserstein_1985} between the real data distribution $P_r\sim R$ and the generated data distribution $P_g\sim G$. In fact, the well-known Wasserstein GAN~\cite{pmlr-Wassersteingan} uses this distance to optimize the GAN models. It is very difficult to compute the Wasserstein distance between two distributions in high dimensions by its original definition. In practice, the \textit{Sliced Wasserstein distance} (SWD)~\cite{bonneel_sliced_2015} is applied to approximate the Wasserstein distance between real and generated images. The key idea of SWD is to obtain several random radial projections of data from high dimensions to one-dimensional spaces and compute their 1-D Wasserstein distances, which have simple solutions~\cite{Ramdas2017Wasserstein,noauthor_scipystatswasserstein_distance_nodate}.

Compared to IS and FID, SWD directly uses the real and generated images without auxiliary networks but it requires that the two data sets have the same number of images: $|R|=|G|$. Usually, the amount of real data is smaller than that of generated data (generated data can be an arbitrarily large amount). And the result of SWD is in general different with each application of the algorithm because of its dimensionality reduction by random projections. Thus, we have to take its average values by computing repeatedly.

As with the FID, the Wasserstein distance measures the difference between distributions of real and generated images and a good GAN can minimize the difference between the two distributions. Hence, for FID and SWD, the smaller value is better.

\subsubsection{Other Evaluations}
As illustrated by the FID and SWD, to compare distributions of real and generated data is an important idea for the GAN evaluation. The \textit{Classifier Two-sample Tests} (C2ST)~\cite{lehmann2006testing} is to examine if two samples belong to the same distribution through a selected classification method. Specifically, any two-class classifier can be employed in the C2ST. To create a C2ST without an additional classifier, \citet{Lopez-Paz2017Revisiting} introduced the \textit{1-Nearest Neighbor Classifier}~(1NNC) measure that uses a two-sample test with the 1-Nearest Neighbor~(1-NN) method on real and generated image sets. Similar to SWD, 1NNC examines whether two distributions of real and generated image are identical and it also requires the numbers of real and generated images to be equal.

Suppose $|R|=|G|$, we apply the \textit{Leave one out cross-validation}~(LOOCV) to a 1-NN classifier trained on dataset: $\{R\cup G\}$ with labels ``1'' for $R$ and ``0'' for $G$. For each validation result, the accuracy is either 1 or 0; and the \textit{Leave-one-out} (LOO) accuracy is the final average of all validation results.
\begin{itemize}
    \item LOO accuracy $\approx0.5$ is the optimal situation because the two distributions are very similar.
    \item LOO accuracy $<0.5$, the GAN is overfitting to $R$ because the generated data are very close to the real samples. In an extreme case, if the GAN memorizes every sample in $R$ and then generates them identically, \textit{i.e.}, $G = R$, the accuracy would be $=0$ because every sample from $R$ would have its nearest neighbor from $G$ with zero distance.
    \item LOO accuracy $>0.5$ means the two distributions are different (separable). If they are completely separable, the accuracy would be $=1$. 
\end{itemize}

Compared to IS and FID, the 1NNC is an independent measure without auxiliary pre-trained classifiers. However, the $|R|=|G|$ requirement limits its applications and \ul{the local conditions of distributions} will greatly affect the 1-NN classifier. For 1NNC, 0.5 is the best score. To compare with other scores, we regularize 1NNC by this function:
\begin{equation} \label{eq:r1nnc}
  r(x)=-|2x-1|+1  
\end{equation}
Let r1NNC $=$ r(1NNC). Therefore, for r1NNC, the best score is 1 and the larger value is better.

As reported by~\citet{Borji2019Pros}, many other GAN evaluation measures have been proposed recently. Measures like the \textit{Average Log-likelihood}~\cite{Theis2016note}, \textit{Coverage Metric}~\cite{tolstikhin2017adagan}, and \textit{Maximum Mean Discrepancy} (MMD)~\cite{JMLR:v13:gretton12a} depend on selected kernels. And measures like the \textit{Classification Performance} (\textit{e.g.}, FCN-score)~\cite{isola2017image}, \textit{Boundary Distortion}~\cite{santurkar2018classification}, \textit{Generative Adversarial Metric} (GAM)~\cite{im2016generating}, \textit{Normalized Relative Discriminative Score} (NRDS)~\cite{zhang2018decoupled}, and \textit{Adversarial Accuracy and Divergence}~\cite{YangKBP17} use various types of auxiliary models. Some measures compare real and generated images based on image-level techniques~\cite{snell2017learning,zeng2017statistics}, such as SSIM, PSNR, and filter responses. The idea of the \textit{Geometry Score} (GS)~\cite{khrulkov2018geometry} is similar to our proposed LS in some aspects but \ul{its results are unstable and rely on required parameters}\footnote{In practice, we used the codes provided by its author: \url{https://github.com/KhrulkovV/geometry-score}.}. We will further discuss the GS in this paper.

By considering the complexity of algorithm, efficiency in high dimensions, dependency on models or parameters, the extent of use in GAN study field, and (codes) availability for implementation, we finally chose the IS, FID, r1NNC(C2ST), MS, AM, and SWD from the currently-used quantitative measures to compare with our proposed LS.

\section{Likeness Score}
Like FID, 1NNC, and SWD, to examine how the distributions of real and generated images are close to each other is an effective way to measure GANs because the goal of GAN training is to make generated images have the same distribution as real ones.

Considering a dataset that contains real and generated data, the most difficult situation to separate the two classes (or two types: real and generated data) of data arises when the two classes are scattered and mixed together in the same distribution. In this sense, the \textit{separability} of real and generated data could be a promising measure of the similarity of the two distributions. As the separability increases, the two distributions have more differences. Therefore, we proposed the \textit{Distance-based Separability Index} (DSI) \footnote{More studies about the DSI will appear in other forthcoming publications, which can be found in author’s website linked up with the ORCID: \url{https://orcid.org/0000-0002-3779-9368}.} to analyze how two classes of data are mixed together.

If a dataset contains data from two classes $X$ and $Y$, the most difficult situation for separation of the dataset occurs when the data of different classes will have the same distribution (distributions have the same shape, position, and support, \textit{i.e.}, the same probability density function). Suppose $X$ and $Y$ have $N_x$  and $N_y$ data points, respectively, we can define:

\begin{definition} \label{def:1}
The \textit{Intra-Class Distance} (ICD) set $\{d_x\}$ is a set of distances between any two points in the same class $(X)$, as: $\{d_x\}=\{ \|x_i-x_j\|_2 | x_i,x_j\in X;x_i\neq x_j\}$.
\end{definition}

\begin{corollary} \label{cor:1}
Given $|X|=N_x$, then $|\{d_x\}|=\frac{1}{2}N_x(N_x-~1)$.
\end{corollary}

\begin{definition} \label{def:2}
The \textit{Between-Class Distance} (BCD) set $\{ d_{x,y}\}$ is the set of distances between any two points from different classes $(X \, and \, Y)$, as $\{ d_{x,y} \}=\{ \|x_i-~y_j\|_2 \, |\, x_i\in X;y_j\in Y \}$.
\end{definition}

\begin{corollary} \label{cor:2}
Given $|X|=N_x,|Y|=N_y$, then $|\{d_{x,y} \}|=N_x N_y$.
\end{corollary}

The metric for all distances is Euclidean $(l^2\,\text{norm})$. In prior work, we have made comparisons for several distance metrics including City-block, Chebyshev, Correlation, Cosine, and Mahalanobis; the Euclidean distance performed best. That is, DSI based on Euclidean distance has the best sensitivity to complexity, and thus we selected it. Then, Theorem \ref{thm:1} shows how the ICD and BCD sets are related to the distributions of the two-class data.

\begin{theorem} \label{thm:1}
When $|\{d_x\}|,|\{d_y\}|\to \infty$, if and only if the two classes $X$ and $Y$ have the same distribution, the distributions of the ICD and BCD sets are identical.
\end{theorem}

The full proof of Theorem \ref{thm:1} is shown in Appendix \ref{sec:proof}. Here we provide an informal explanation: points in $X$ and $Y$ having the same distribution can be considered to have been sampled from one distribution $Z$. Hence, both ICDs of $X$ and $Y$, and BCDs between $X$ and $Y$ are actually ICDs of $Z$. Consequently, the distributions of ICDs and BCDs are identical. In other words, that the distributions of the ICD and BCD sets are identical indicates all labels are assigned randomly and thus, the dataset has the least separability. And if the distributions of the ICD and BCD sets are nearly identical, we see that their histograms are almost overlapped. That is, for a distance $d$:
\[
\frac{|\{d_x=d\}|}{|\{d_x\}|} \approx \frac{|\{d_y=d\}|}{|\{d_y\}|} \approx \frac{|\{d_{x,y}=d\}|}{|\{d_{x,y} \}|}
\]

The time costs for computing the ICD and BCD sets increase linearly with the number of dimensions and quadratically with the amount of data. In practice, the time costs could be greatly reduced by using parallel computing.

\subsection{Computation of DSI for GANs Evaluation}
\label{ls4gan}
Since for GANs' evaluation, there are only \ul{two classes}: the real image set $R$ and generated image set $G$, we have two ICD sets and one BCD set. In fact, the DSI can be applied in a multi-class scenario by one-\textit{versus}-others; the process is shown in Appendix \ref{sec:multiclass}. Here we focus on the computation of DSI for GANs' evaluation (two-class scenario).

\textbf{First}, the ICD sets of $R$ and $G$: $\{d_r\},\{d_g\}$ and the BCD set: $\{d_{r,g}\}$ are computed by their definitions (Def.~\ref{def:1}~and~\ref{def:2}). 

\textbf{Second}, to examine the similarity of the distributions of the ICD and BCD sets, we apply the Kolmogorov–Smirnov (KS) distance \citep{Kolmogorov1951}: 
\[
s_r=KS(\{d_r\},\{d_{r,g}\}) \text{, and } s_g=KS(\{d_g\},\{d_{r,g}\}). 
\]
The result of a two-sample KS distance\footnote{In experiments, we used the \texttt{scipy.stats.ks\_2samp} from the SciPy package in Python to compute the KS distance. \url{https://docs.scipy.org/doc/scipy/reference/generated/scipy.stats.ks_2samp.html}} is the maximum distance between two cumulative distribution functions (CDFs):
\[
KS(P,Q)=\sup_{x} |P(x)-Q(x)|
\]
Where $P$ and $Q$ are the respective CDFs of the two distributions $p$ and $q$. Although there are other statistical measures to compare two distributions, such as Bhattacharyya distance, Kullback–Leibler divergence, and Jensen–Shannon divergence, most of them require the two sets to have the same number of samples. It is easy to show that the $|\{d_r\}|,|\{d_g\}|$ and $|\{d_{r,g}\}|$ cannot be the same. The Wasserstein distance is also a potentially suitable measure but we found that it is not as sensitive as the KS distance.

\textbf{Finally}, the DSI for GANs evaluation is the maximum of two KS distances: 
\[
DSI(\{R,G\})=\max\{s_r,\ s_g\},
\]
because the maximum value can highlight the difference between ICD and BCD sets. The similarity of the distributions of the ICD sets: $KS(\{d_r\},\{d_g\})$ is not used because it shows only the difference of distribution shapes, not their location information. For example, two distributions that have the same shape but no overlap will have \ul{zero} KS distance between their ICD sets: $KS(\{d_r\},\{d_g\})=0$.

Fig.~\ref{fig:2} displays artificial 2D examples of generated data (orange points; blue points are real data) that respectively lack creativity, diversity, and inheritance. With respect to the ICD and BCD sets, if the generated data overfit the real data (lack of creativity), peaks will appear in the distribution of BCD near zero (see Fig.~\ref{fig:2}a) because there are many generated points that are close to real data points in their distribution space; hence, many BCD are close to zero. Similarly, lack of diversity implies that many generated data points are close to each other; thus, many ICD values are close to zero and peaks will appear in the distribution of ICD near zero (see Fig.~\ref{fig:2}b). Lack of inheritance is shown by the difference between the distributions of ICD and BCD (see Fig.~\ref{fig:2}c) because if and only if the two classes (real data and generated data) have the same distribution, the distributions of ICD and BCD sets are identical. In that case, there is neither lack of creativity nor lack of diversity. This is because there will be \ul{no single peaks of ICD or BCD near zero}. Therefore, the DSI well evaluates the GAN's performance by measuring creativity, diversity, and inheritance.

DSI ranges from 0 to 1; a small DSI (low separability) means that the ICD and BCD sets are very similar, and by Theorem \ref{thm:1}, the distributions of real and generated data are similar too. Hence, the GAN performs well. \ul{To be consistent with other comparison measures, we complement its value and define the \textit{Likeness Score} (LS)}:
\[
\text{LS}=1-\text{DSI},
\]
which is closer to 1 if the GAN performs better.

\begin{figure}[t]
    \includegraphics[width=.48\textwidth]{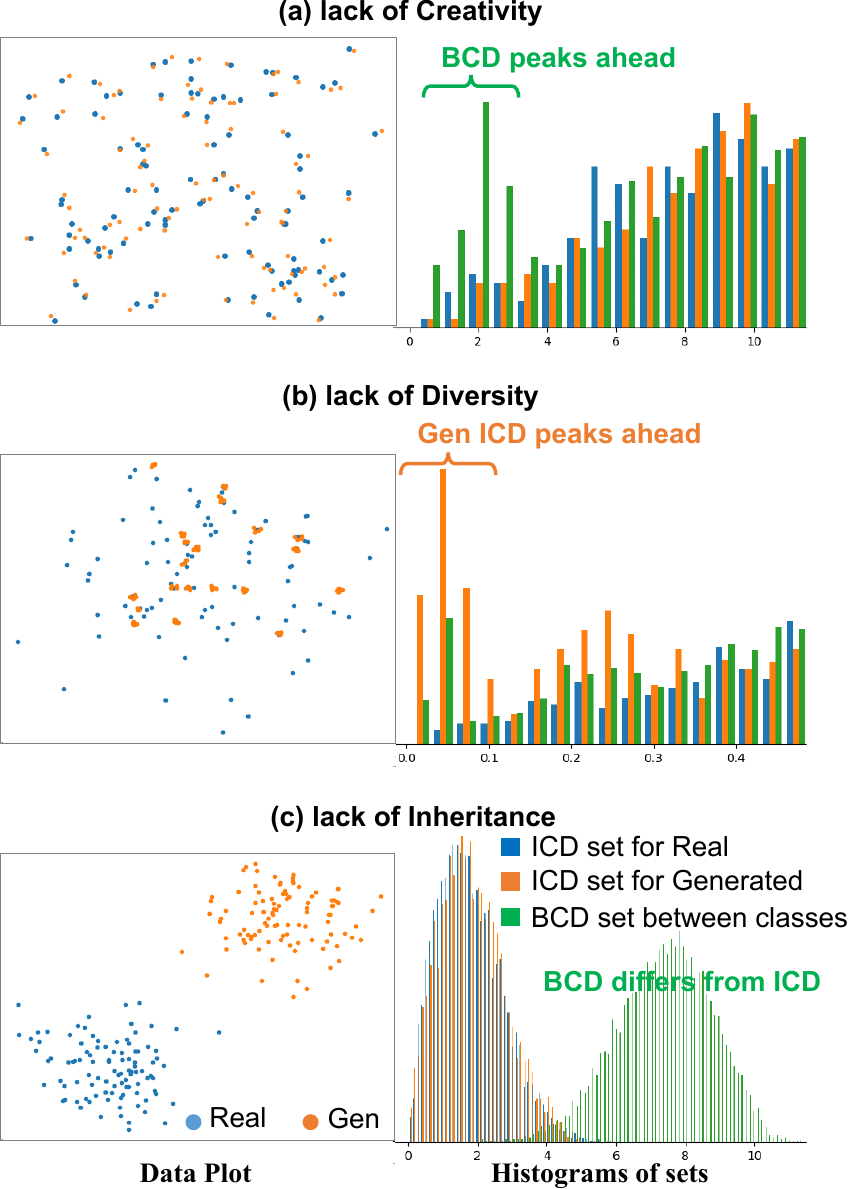}
    \caption{Lack of Creativity, Diversity, and Inheritance in 2D. Histograms of (a) and (b) are zoomed to ranges near zero; (c) has the entire histogram.}
    \label{fig:2}
\end{figure}

\section{Experiments \& Results}
The first experiment has two purposes: one is to test the stability of the proposed measure, \textit{i.e.}, how little the results change when different amounts of data are used. Another purpose is to find the minimum amount of data required for the following experiments because a GAN could generate unlimited data and we wish to bound it to make computation practicable.

The following experiments compare our measure LS with the commonly used measures: IS and FID, and other selected measures. The purpose is not to show which GAN is better but to show how the results (values) of our measure compare to those of existing measures.

\subsection{One Image Type by DCGAN}

\begin{table}[h]
\caption{Measure values for different numbers of generated images}
\label{tab:1}
\resizebox{.48\textwidth}{!}{
\begin{threeparttable}
\begin{tabular}{cc:ccccccc}
    \hline\noalign{\smallskip}
    \textbf{\#} & \textbf{LS} & \textbf{IS} & \textbf{FID} & \textbf{r1NNC$^\dagger$} & \textbf{MS} & \textbf{AM} & \textbf{SWD} & \textbf{GS} \\
    \noalign{\smallskip}\hline\noalign{\smallskip}
    120 & 0.613 & 1.435 & 148.527 & 0.850 & 0.791 & 456.660 & 717.471 & 0.311 \\
    240 & 0.644 & 1.424 & 134.484 & 0.858 & 0.809 & 456.119 & 673.341 & 0.757 \\
    480 & 0.636 & 1.409 & 135.317 & 0.821 & 0.834 & 451.786 & 668.462 & 1.074 \\
    960 & 0.622 & 1.447 & 145.142 & 0.833 & 0.852 & 451.338 & 667.519 & 0.908 \\
    1200 & 0.630 & 1.426 & 141.818 & 0.862 & 0.827 & 454.656 & 675.751 & 1.000 \\
    2400 & 0.628 & 1.431 & 146.109 & 0.850 & 0.844 & 452.077 & 685.621 & 0.454 \\
    4800 & 0.622 & 1.440 & 145.109 & 0.851 & 0.842 & 451.255 & 678.986 & 0.526 \\
    \noalign{\smallskip}\hline
\end{tabular}
    \begin{tablenotes}
      \item Dashed line: to the left are our proposed measures; to the right are compared measures.
      \item $^\dagger$ r1NNC is the regularized 1NNC, defined by Eq. \ref{eq:r1nnc}.
    \end{tablenotes}
\end{threeparttable}}
\end{table}

To test the proposed measures, in the first experiment, we used one type of image (Plastics; 12 images) from the USPtex database~\cite{Backes2012Color} to train a DCGAN. Then, the trained GAN generated several groups containing different amounts of synthetic images. Finally, we compute results of our proposed measure (LS), IS, FID, r1NNC, MS, AM, SWD and GS by using these generated images and 12 real images; the results are shown in Table~\ref{tab:1}.

Computations of FID, r1NNC and SWD require that the two image sets have the same number of images. We divided the generated images into many 12-image subsets to compute the scores with 12 real images and then obtained their average values. Fig.~\ref{fig:3} shows the plots of these scores. To fit the axes, the values of FID, AM, and SWD are scaled by 0.01, 0.001, and 0.001, respectively. The result indicates that the scores \ul{except the GS}, are stable to different numbers of testing images, especially when the amount is greater than 1000. We remove the GS from further comparisons because its results are highly unstable with the amount of data.

\begin{figure}[t]
    \includegraphics[width=.48\textwidth]{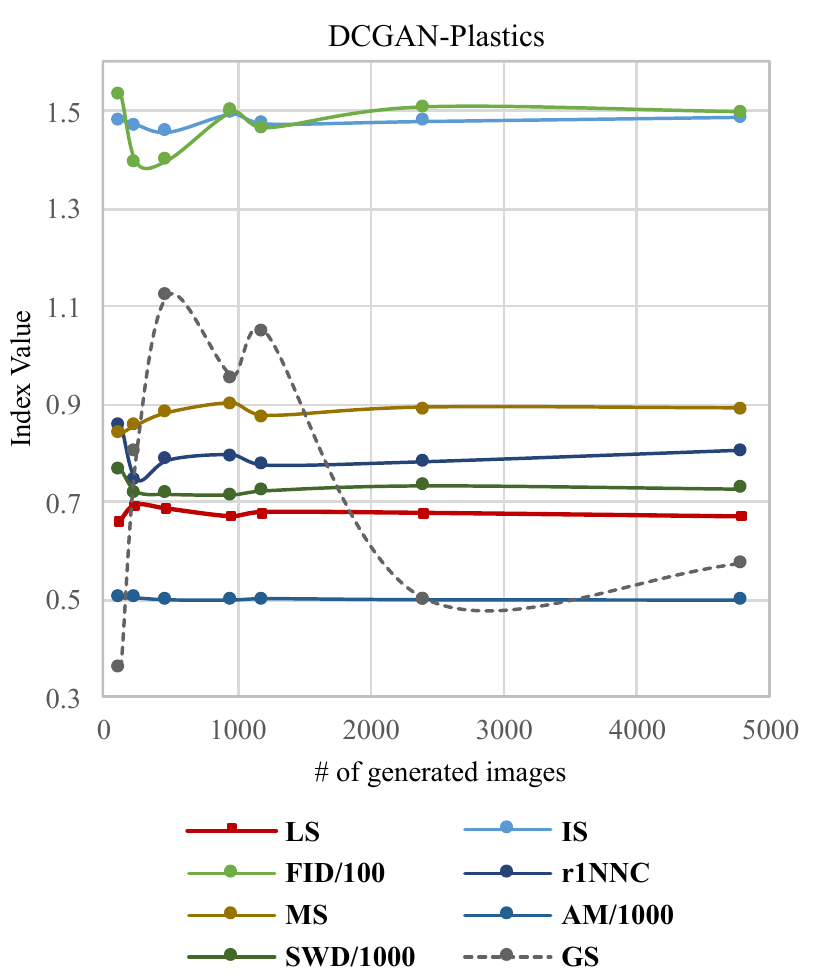}
    \caption{Plots of values in Table~\ref{tab:1}.}
    \label{fig:3}
\end{figure}

\subsection{Four Image Types and Three GANs}
In the second experiment, four types of image (Holes, Small leaves, Big leaves, and Plastics; 12 images for each type) are used to train three GANs (DCGAN, WGAN-GP, and SNGAN). Then, the trained GANs generated 1,200 synthetic images for each type. Twelve sets of synthetic images were generated; Fig.~\ref{fig:4} shows samples from 4 real image sets and 12 generated image sets. Visual examination of these synthetic images indicates that the DCGAN seems to give the most images similar to the real ones, but many of its generated images are duplications of real ones. Thus, the DCGAN overfitted the training data. The SNGAN’s generated images are most dissimilar from real images; they lack the inheritance feature. The WGAN-GP well balanced the creativity and inheritance features.

\begin{figure}[h]
\begin{tabular}{p{0.5em}@{\hspace{0.5em}}c@{\hspace{0.5em}}c@{\hspace{0.5em}}c@{\hspace{0.5em}}c}
  & \smallbf{Real} & \smallbf{DCGAN} & \smallbf{WGAN-GP} & \smallbf{SNGAN} \\ 
  \rotatebox{90}{\smallbf{\hspace{1.3em} Hole}} & \includegraphics[width=.21\linewidth]{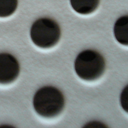} & \includegraphics[width=.21\linewidth]{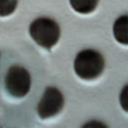} & \includegraphics[width=.21\linewidth]{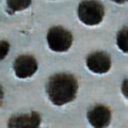} & \includegraphics[width=.21\linewidth]{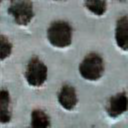}
  \\ 
  \rotatebox{90}{\smallbf{\hspace{0.2em} Small leaf}} & \includegraphics[width=.21\linewidth]{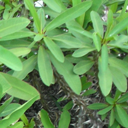} & \includegraphics[width=.21\linewidth]{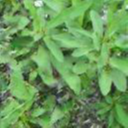} & \includegraphics[width=.21\linewidth]{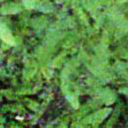} & \includegraphics[width=.21\linewidth]{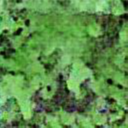}
 \\ 
   \rotatebox{90}{\smallbf{\hspace{0.5em} Big leaf}} & \includegraphics[width=.21\linewidth]{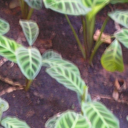} & \includegraphics[width=.21\linewidth]{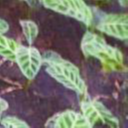} & \includegraphics[width=.21\linewidth]{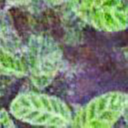} & \includegraphics[width=.21\linewidth]{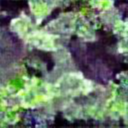}
 \\ 
   \rotatebox{90}{\smallbf{\hspace{0.8em} Plastic}} & \includegraphics[width=.21\linewidth]{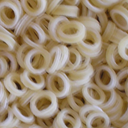} & \includegraphics[width=.21\linewidth]{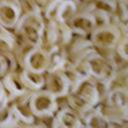} & \includegraphics[width=.21\linewidth]{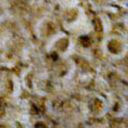} & \includegraphics[width=.21\linewidth]{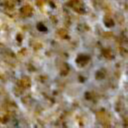}
\end{tabular}
    \caption{Column 1: samples from four types of real images; column 2-4: samples from synthetic images of three GANs trained by the four types of images.}
    \label{fig:4}
\end{figure}

We applied these measures on the 12 generated image sets; results are shown in Table~\ref{tab:2}. Fig.~\ref{fig:5} shows plots of results. To emphasize the rank of each score for different generators and image types, values are normalized and ranked from 0 to 1 by columns for plotting; 0 is for the worst (model) performance and 1 is for the best (model) performance. Table~\ref{tab:3} averaged scores by GAN models. To compare the three GANs, Table~\ref{tab:3} shows summarized results and Fig.~\ref{fig:5} gives more details. In general, the absolute values of measures are not significant but their ranks matter because for infinite-range measures, such as IS, FID, and SWD, their values highly depend on the input data. Therefore, little importance should be attached to their differences.

\begin{table}[h]
\caption{Measure results}
\label{tab:2}
\resizebox{.48\textwidth}{!}{
\begin{threeparttable}
\begin{tabular}{lc:cccccc}
    \hline\noalign{\smallskip}
    \textbf{*} & \textbf{LS} & \textbf{IS} & \textbf{FID$^\downarrow$} & \textbf{r1NNC} & \textbf{MS} & \textbf{AM$^\downarrow$} & \textbf{SWD$^\downarrow$} \\
    \noalign{\smallskip}\hline\noalign{\smallskip}
    DC-h & 0.747 & 1.222 & 102.805 & 0.892 & 0.866 & 407.841 & 862.241 \\
    DC-sl & 0.611 & 1.171 & 155.973 & 0.858 & 0.934 & 511.218 & 944.228 \\
    DC-bl & 0.262 & 1.321 & 172.296 & 1.000 & 0.573 & 509.649 & 1053.687 \\
    DC-pla & 0.630 & 1.426 & 141.818 & 0.908 & 0.827 & 454.656 & 678.210 \\ \hdashline[.5pt/1pt]
    W-h & 0.771 & 1.163 & 233.277 & 0.958 & 0.671 & 607.249 & 604.263 \\
    W-sl & 0.465 & 1.369 & 400.036 & 0.983 & 0.155 & 726.232 & 702.976 \\
    W-bl & 0.626 & 1.536 & 375.987 & 0.975 & 0.117 & 779.834 & 650.157 \\
    W-pla & 0.441 & 1.555 & 513.268 & 0.792 & 0.026 & 1108.549 & 732.241 \\ \hdashline[.5pt/1pt]
    SN-h & 0.594 & 1.317 & 252.857 & 1.000 & 0.467 & 570.819 & 778.487 \\
    SN-sl & 0.025 & 1.105 & 469.795 & 0.133 & 0.158 & 879.136 & 1110.309 \\
    SN-bl & 0.000 & 1.083 & 456.813 & 0.195 & 0.077 & 1086.094 & 1221.202 \\
    SN-pla & 0.000 & 1.037 & 485.716 & 0.000 & 0.032 & 1399.649 & 1229.506 \\
    \noalign{\smallskip}\hline
\end{tabular}
    \begin{tablenotes}
      \item *Generator models: DC: DCGAN, W: WGAN-GP, SN: SNGAN. Generated image types: h: hole, sl: small leaf, bl: big leaf, pla: plastic.
      \item Dashed line: to the left are our proposed measures; to the right are the compared measures.
      \item $^\downarrow$ Measures with this symbol mean smaller score is better; otherwise, larger score is better.
    \end{tablenotes}
\end{threeparttable}}
\end{table}

\begin{figure}[t]
    \includegraphics[width=.48\textwidth]{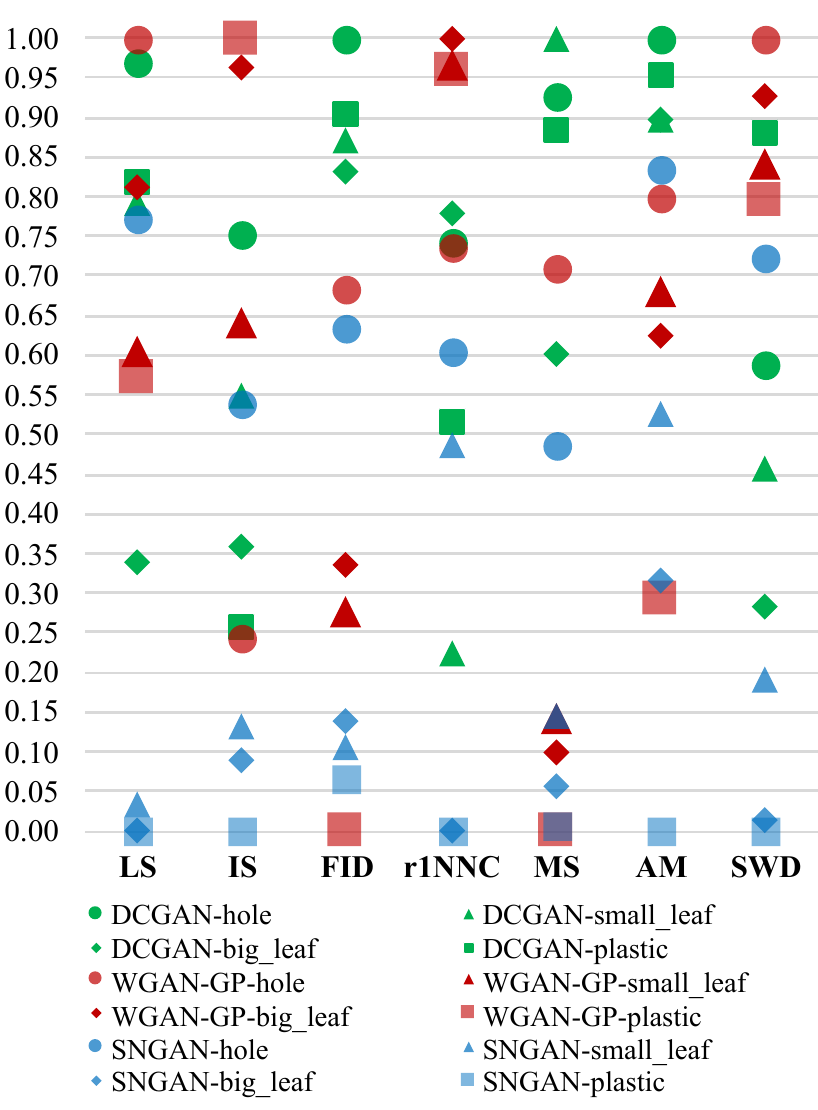}
    \caption{Normalized and ranked scores. X-axis shows scores and y-axis shows their normalized values; 0 is for the worst (model) performance and 1 is for the best (model) performance. Colors are for generators and shapes are for image types; see details in legend.}
    \label{fig:5}
\end{figure}

\begin{table}[h]
\caption{Measure results averaged by generators}
\label{tab:3}
\resizebox{.48\textwidth}{!}{
\begin{threeparttable}
\begin{tabular}{lc:cccccc}
    \hline\noalign{\smallskip}
    \textbf{Model} & \textbf{LS} & \textbf{IS} & \textbf{FID$^\downarrow$} & \textbf{r1NNC} & \textbf{MS} & \textbf{AM$^\downarrow$} & \textbf{SWD$^\downarrow$} \\
    \noalign{\smallskip}\hline\noalign{\smallskip}
    DCGAN & 0.562 & 1.285 & \textbf{143.223} & 0.915 & \textbf{0.800} & \textbf{470.841} & 884.592 \\
    WGAN-GP & \textbf{0.576} & \textbf{1.406} & 380.642 & \textbf{0.927} & 0.242 & 805.466 & \textbf{672.409} \\
    SNGAN & \underline{0.155} & \underline{1.135} & \underline{416.295} & \underline{0.332} & \underline{0.184} & \underline{983.924} & \underline{1084.876} \\
    \noalign{\smallskip}\hline
\end{tabular}
    \begin{tablenotes}
      \item \textbf{Bold value}: the best model by the measure of this column.
      \item \underline{Underline}: the worst model by the measure of this column.
      \item Dashed line: to the left are our proposed measures; to the right are the compared measures.
      \item $^\downarrow$ Measures with this symbol mean smaller score is better; otherwise, larger score is better.
    \end{tablenotes}
\end{threeparttable}}
\end{table}

For the best generator, the proposed LS agrees with IS, 1NNC, SWD, and the visual appearance of generated images. Since the DCGAN overfitted to training data, it lacks creativity, but FID, MS, and AM rank it as the best model. All measures including the LS rank SNGAN as the worst because it \ul{lacks diversity}. Especially, for the SNGAN-\texttt{big leaf} and SNGAN-\texttt{plastic} whose LS values are zero (in Table~\ref{tab:2}), almost all images are the same (but different from real ones).

\subsection{Five GANs on CIFAR-10}
\label{cifar-10}
\begin{table}[h]
\caption{Measure results on CIFAR-10}
\label{tab:4}
\resizebox{.48\textwidth}{!}{
\begin{threeparttable}
\begin{tabular}{lc:cccccc}
    \hline\noalign{\smallskip}
    \textbf{Model} & \textbf{LS} & \textbf{IS} & \textbf{FID$^\downarrow$} & \textbf{r1NNC} & \textbf{MS} & \textbf{AM$^\downarrow$} & \textbf{SWD$^\downarrow$} \\
    \noalign{\smallskip}\hline\noalign{\smallskip}
    DCGAN & 0.833 & \textbf{4.311} & 147.110 & 0.772 & \textbf{1.878} & \textbf{335.879} & 710.993 \\
    WGAN-GP & \textbf{0.957} & 3.408 & \textbf{136.121} & \textbf{0.932} & 1.483 & 507.374 & \textbf{276.189} \\
    SNGAN & \underline{0.593} & \underline{2.049} & \underline{219.762} & \underline{0.534} & 0.860 & \underline{631.807} & \underline{743.679} \\
    LSGAN & 0.745 & 3.405 & 136.132 & 0.716 & 1.337 & 450.250 & 710.747 \\
    SAGAN & 0.688 & 2.075 & 206.046 & 0.545 & \underline{0.814} & 611.706 & 595.761 \\
    \noalign{\smallskip}\hline
\end{tabular}
    \begin{tablenotes}
      \item \textbf{Bold value}: the best model by the measure of this column.
      \item \underline{Underline}: the worst model by the measure of this column.
      \item Dashed line: to the left are our proposed measures; to the right are the compared measures.
      \item $^\downarrow$ Measures with this symbol mean smaller score is better; otherwise, larger score is better.
    \end{tablenotes}
\end{threeparttable}}
\end{table}

In the third experiment, we used the CIFAR-10 dataset that is widely used in machine learning to train more types of GANs (DCGAN, WGAN-GP, SNGAN, LSGAN, and SAGAN). A 2,000-image subset had been chosen randomly from the training set of CIFAR-10 to train the five GANs. Five sets of synthetic images were generated; Fig.~\ref{fig:6} shows samples from the original 2,000-image subset and five generated image sets.

Then, each trained GAN generated 2,000 synthetic images and we applied the LS, and other six measures to the five generated image sets and the original 2,000-image subset. Results are shown in Table~\ref{tab:4}. LS agrees with FID, 1NNC, and SWD that WGAN-GP is the best GAN model but IS, MS, and AM rank DCGAN as the best model. For the worst model, LS agrees with all the other measures except the MS. MS shows the SAGAN performs worst but the MS scores of SAGAN and SNGAN are small and close.

\begin{figure}[h]
\begin{tabular}{c@{\hspace{0.5em}}c@{\hspace{0.5em}}c@{\hspace{0.5em}}c@{\hspace{0.5em}}c@{\hspace{0.5em}}c}
   \smallbf{Real} & \smallbf{DC} & \smallbf{W} & \smallbf{SN} & \smallbf{LS} & \smallbf{SA} \\ 
  \includegraphics[width=.14\linewidth]{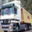} & \includegraphics[width=.14\linewidth]{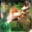} & \includegraphics[width=.14\linewidth]{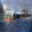} & \includegraphics[width=.14\linewidth]{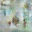} & \includegraphics[width=.14\linewidth]{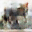} &
  \includegraphics[width=.14\linewidth]{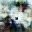} \\
  
    \includegraphics[width=.14\linewidth]{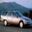} & \includegraphics[width=.14\linewidth]{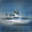} & \includegraphics[width=.14\linewidth]{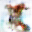} & \includegraphics[width=.14\linewidth]{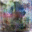} & \includegraphics[width=.14\linewidth]{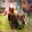} &
  \includegraphics[width=.14\linewidth]{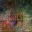} \\
  
    \includegraphics[width=.14\linewidth]{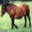} & \includegraphics[width=.14\linewidth]{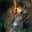} & \includegraphics[width=.14\linewidth]{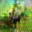} & \includegraphics[width=.14\linewidth]{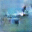} & \includegraphics[width=.14\linewidth]{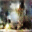} &
  \includegraphics[width=.14\linewidth]{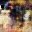} \\
  
    \includegraphics[width=.14\linewidth]{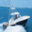} & \includegraphics[width=.14\linewidth]{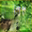} & \includegraphics[width=.14\linewidth]{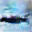} & \includegraphics[width=.14\linewidth]{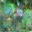} & \includegraphics[width=.14\linewidth]{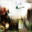} &
  \includegraphics[width=.14\linewidth]{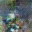} \\
  
    \includegraphics[width=.14\linewidth]{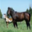} & \includegraphics[width=.14\linewidth]{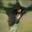} & \includegraphics[width=.14\linewidth]{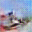} & \includegraphics[width=.14\linewidth]{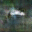} & \includegraphics[width=.14\linewidth]{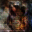} &
  \includegraphics[width=.14\linewidth]{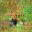} \\
  
    \includegraphics[width=.14\linewidth]{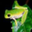} & \includegraphics[width=.14\linewidth]{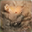} & \includegraphics[width=.14\linewidth]{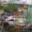} & \includegraphics[width=.14\linewidth]{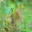} & \includegraphics[width=.14\linewidth]{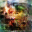} &
  \includegraphics[width=.14\linewidth]{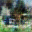}
  
\end{tabular}
    \caption{Column 1: samples from real images of CIFAR-10; column 2-6: samples from synthetic images of five GANs: DCGAN, WGAN-GP, SNGAN, LSGAN, and SAGAN trained by the original 2,000-image subset.}
    \label{fig:6}
\end{figure}

\subsection{Virtual GANs on MNIST}
To emphasize the measurements of creativity, diversity, and inheritance, in the fourth experiment, we created five artificial image sets to \ul{simulate} the optimal generated images and generated images that lack creativity, lack diversity, lack both creativity and diversity, and lack inheritance. Images are taken or modified from the MNIST database~\cite{LeCun2010MNIST}, which contains $28\times28$-pixel handwritten-digit images with labels $\{0,\ 1,\ 2,\ \cdots,\ 9\}$. Fig.~\ref{fig:7} describes how the five artificial sets were built.

\begin{figure}[t]
    \includegraphics[width=.49\textwidth]{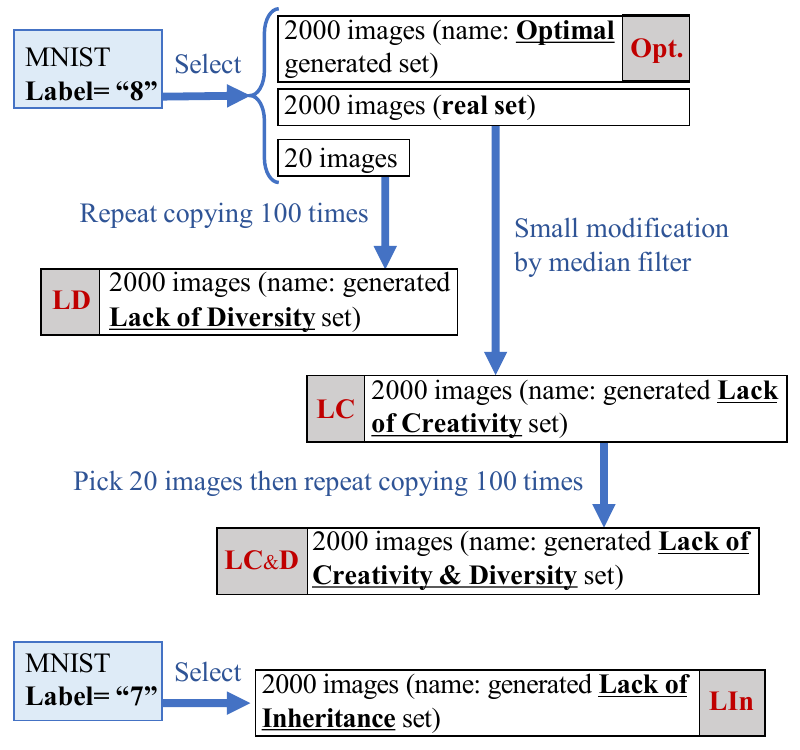}
    \caption{Processes to build real set and generated sets including optimal generated images and generated images lack creativity, lack diversity, lack creativity \& diversity, and lack inheritance.}
    \label{fig:7}
\end{figure}

Three subsets containing 2,000, 2,000, and 20 images were randomly selected from handwritten digit ``8'' images in the MNIST database. There is no common image in the three sets. One set having 2,000 images was considered as the \textit{optimal generated set} (Opt.) because these images come from the same source of real data. The \textit{lack-of-diversity set} (LD) was generated by repeatedly copying the 20 images 100 times. 
Another 2,000-image set was considered as the \textit{real set} and used to generate the \textit{lack-of-creativity set} (LC) by the small modification of all images with the median filter. Since filtering could slightly change images and keep their main information, each image after filtering is similar to its original version \textit{i.e.}, the modified images lack creativity. Choosing 20 images from the lack-of-creativity set and repeatedly copying them 100 times generates the \textit{lack-of-creativity \& diversity set} (LC{\small \&}D). The \textit{lack-of-inheritance set} (LIn) contains 2,000 images selected randomly from handwritten digit ``7'' images in MNIST because the handwritten digit ``7'' is greatly different from digit ``8''.

\begin{figure}[h]
\begin{tabular}{c@{\hspace{0.5em}}c@{\hspace{0.5em}}c@{\hspace{0.5em}}c@{\hspace{0.5em}}c@{\hspace{0.5em}}c}
\smallbf{Real} & \smallbf{Opt.} & \smallbf{LC} & \smallbf{LD} & \smallbf{LC\&D} & \smallbf{LIn} \\
  \includegraphics[width=.14\linewidth]{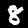} & \includegraphics[width=.14\linewidth]{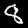} & \includegraphics[width=.14\linewidth]{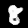} & \includegraphics[width=.14\linewidth]{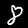} & \includegraphics[width=.14\linewidth]{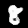} &
  \includegraphics[width=.14\linewidth]{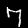} \\
  
  \includegraphics[width=.14\linewidth]{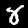} & \includegraphics[width=.14\linewidth]{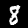} & \includegraphics[width=.14\linewidth]{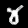} & \includegraphics[width=.14\linewidth]{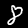} & \includegraphics[width=.14\linewidth]{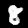} &
  \includegraphics[width=.14\linewidth]{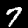} \\
  
  \includegraphics[width=.14\linewidth]{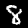} & \includegraphics[width=.14\linewidth]{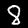} & \includegraphics[width=.14\linewidth]{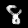} & \includegraphics[width=.14\linewidth]{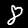} & \includegraphics[width=.14\linewidth]{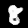} &
  \includegraphics[width=.14\linewidth]{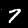} \\
  
  \includegraphics[width=.14\linewidth]{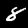} & \includegraphics[width=.14\linewidth]{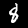} & \includegraphics[width=.14\linewidth]{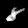} & \includegraphics[width=.14\linewidth]{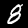} & \includegraphics[width=.14\linewidth]{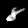} &
  \includegraphics[width=.14\linewidth]{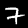} \\
  
  \includegraphics[width=.14\linewidth]{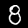} & \includegraphics[width=.14\linewidth]{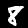} & \includegraphics[width=.14\linewidth]{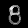} & \includegraphics[width=.14\linewidth]{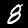} & \includegraphics[width=.14\linewidth]{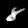} &
  \includegraphics[width=.14\linewidth]{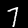} \\
  
  \includegraphics[width=.14\linewidth]{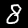} & \includegraphics[width=.14\linewidth]{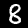} & \includegraphics[width=.14\linewidth]{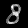} & \includegraphics[width=.14\linewidth]{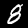} & \includegraphics[width=.14\linewidth]{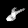} &
  \includegraphics[width=.14\linewidth]{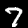}
  
\end{tabular}
    \caption{Column 1: samples from the real set; column 2-6: sample images from the five virtual GAN models: Opt., LC, LD, LC{\small \&}D, and LIn trained by the real set.}
    \label{fig:8}
\end{figure}

The five datasets: Opt., LC, LD, LC{\small \&}D, and LIn \ul{mimic} the datasets that are generated from five virtual GAN models trained on the 2,000-image \textit{real set}. The optimal generated set (Opt.) as if it was generated from an optimal GAN and the other four sets as if they were generated from four different GANs having respective drawbacks. Fig.~\ref{fig:8} shows samples from these datasets.
Then, we applied the LS, and other six measures to the five ``generated'' image sets and the 2,000-image real set. Results are shown in Table~\ref{tab:5}.

\begin{table}[h]
\caption{Measure results from virtual GAN models}
\label{tab:5}
\resizebox{.48\textwidth}{!}{
\begin{threeparttable}
\begin{tabular}{lc:cccccc}
    \hline\noalign{\smallskip}
    \textbf{Model} & \textbf{LS} & \textbf{IS} & \textbf{FID$^\downarrow$} & \textbf{r1NNC} & \textbf{MS} & \textbf{AM$^\downarrow$} & \textbf{SWD$^\downarrow$} \\
    \noalign{\smallskip}\hline\noalign{\smallskip}
    \textbf{Opt.}& \textbf{0.994} & 1.591 & \textbf{4.006} & \textbf{0.978} & \textbf{1.968} & 343.842 & \textbf{23.427} \\
    LC & 0.820 & \textbf{2.112} & 67.310 & 0.039 & 1.007 & 371.322 & 657.527 \\
    LD & 0.892 & \underline{1.299} & 59.112 & \underline{0.002} & 1.597 & \textbf{337.553} & 211.140 \\
    LC{\small \&}D & 0.775 & 1.418 & 116.656 & 0.775 & 0.789 & 389.437 & 740.512 \\
    LIn & \underline{0.526} & 1.941 & \underline{130.827} & 0.462 & \underline{0.605} & \underline{441.292} & \underline{1166.082} \\
    \noalign{\smallskip}\hline
\end{tabular}
    \begin{tablenotes}
      \item \textbf{Bold value}: the best model by the measure of this column.
      \item \underline{Underline}: the worst model by the measure of this column.
      \item Dashed line: to the left are our proposed measures; to the right are the compared measures.
      \item $^\downarrow$ Measures with this symbol mean smaller score is better; otherwise, larger score is better.
    \end{tablenotes}
\end{threeparttable}}
\end{table}

In this experiment, we know the Opt. GAN is the best one. Hence, we could state the concrete conclusion that LS, FID, 1NNC, MS, and SWD successfully discover the best GAN model. As we discussed in Section~\ref{related}, results of IS confirm that it is not good at evaluating the creativity and inheritance of GANs because it gives them higher scores (2.112 and 1.941) than the best case (1.591) and the IS emphasizes the diversity. Other measures also show their characteristics and preferences: LS agrees with FID, MS, MA, and SWD that the worst model is lack of inheritance; IS and 1NNC indicate that the model lacking diversity is the worst. By contrast, AM does not care about the diversity very much because its scores of the best model and the model lacking diversity are similar; and LS, FID, MS, AM, and SWD value creativity more among diversity and creativity.

\section{Discussion}

\begin{figure*}[htbp]
    \centering
    \includegraphics[width=.95\textwidth]{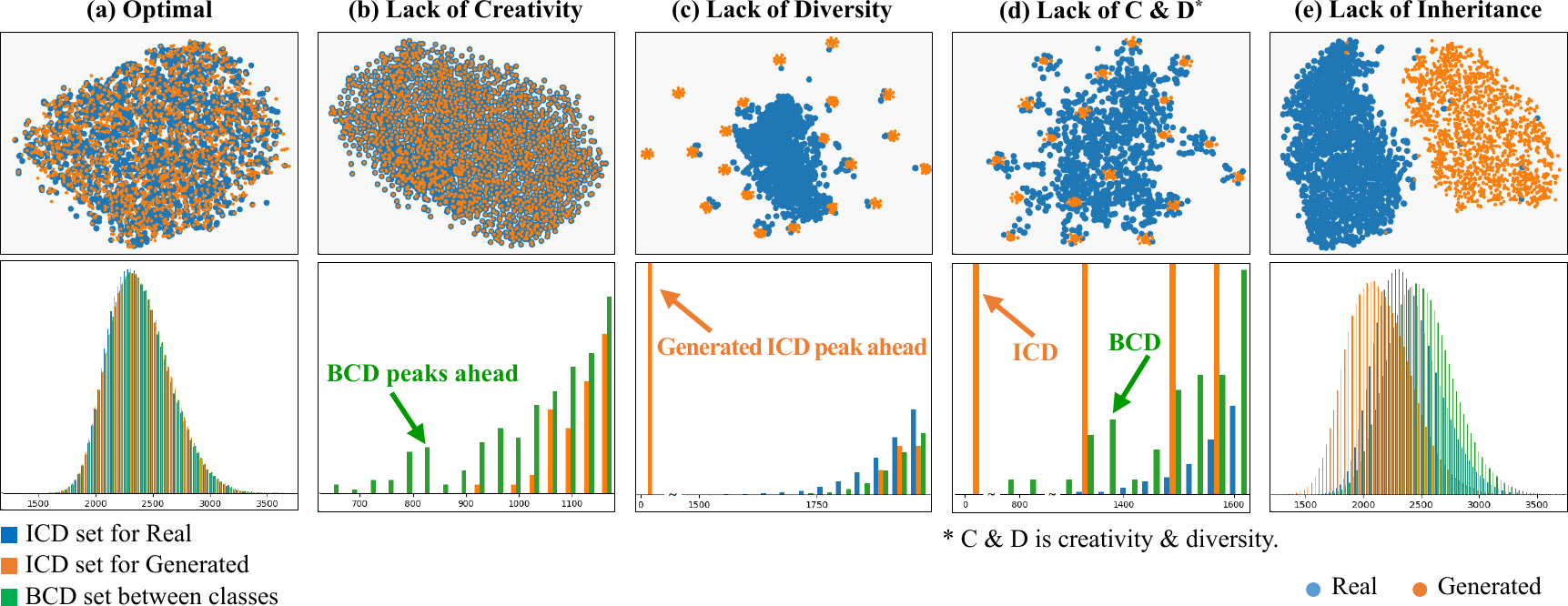}
    \caption{Real and generated datasets from virtual GANs on MNIST. First row: the 2D tSNE plots of real (blue) and generated (orange) data points from each virtual GAN. Second row: histograms of ICDs (blue for real data; orange for generated data) and BCD for real and generated datasets. The histograms in (b)-(d) are zoomed to the beginning of plots; (a) and (e) have the entire histograms.}
    \label{fig:9}
\end{figure*}

Since \citet{Geirhos2019ImageNet-trained} recently reported that CNNs trained by ImageNet have a strong bias to recognize textures rather than shapes, we chose texture images to train GANs. From results in Table~\ref{tab:3}, the proposed LS agrees with IS, 1NNC, and SWD that the WGAN-DP performs the best and SNGAN performs the worst on selected texture images. As shown in Table~\ref{tab:4}, LS makes the same evaluation on CIFAR-10 dataset. As shown in Fig.~\ref{fig:4}, SNGAN and WGAN-GP generate synthetic images that look different from real samples but SNGAN tends to generate many very similar images (its diversity is low). Hence, all measures rate SNGAN as performing worst on texture datasets. Results on CIFAR-10 dataset (Table~\ref{tab:4}) show a similar conclusion.

\subsection{Evaluation of GAN Measures}
Our results indicate that LS is a promising measure for GAN. Without a gold standard, however, it is difficult to compare GAN evaluation methods and to state which method is better when they performed similarly. 
To show measures' characteristics/preferences and evaluate them in terms of the three respects of an ideal GAN, we artificially created five datasets (Fig.~\ref{fig:7}) as if they were generated from five virtual GANs trained on MNIST. In this controlled circumstance, the LS, FID, 1NNC, MS, and SWD discerned the best GAN model (Table~\ref{tab:5}). In addition, by analyzing the distributions of ICD and BCD sets, LS could provide evidences for the lack of creativity, diversity, and inheritance to \ul{explain} its results. 
As with Fig.~\ref{fig:2}, we plot data and histograms of their ICD and BCD sets in Fig.~\ref{fig:9} to show their relationships with the LS. Each image in MNIST has $28\times 28$ pixels so that these data are in a 784-dimensional space. To visually represent the data in two dimensions, we applied the \textit{t-distributed Stochastic Neighbor Embedding} (tSNE)~\cite{Maaten2008Visualizing} method. In contrast, the ICD and BCD sets were computed in the 784-dimensional space directly, without using any dimensionality reduction or embedding methods.

As shown in Fig.~\ref{fig:9}, the ICD and BCD sets for computing the LS offer an interpretation of how LS works and verify that LS is able to detect the lack of creativity, diversity, and inheritance for GAN generated data, as we discussed in Section~\ref{ls4gan}. Fig.~\ref{fig:9}(a) shows the real (training) data and data generated by the ideal GAN. Since distributions of the three sets are nearly the same, LS gets the highest score (close to 1, in Table~\ref{tab:5}). Fig.~\ref{fig:9}(b) shows the GAN lacks creativity. Almost every generated data point is overlapped with (or very close to) a real data point. Hence, the BCD set has some peaks at the beginning of plot. Lack of diversity is shown by Fig.~\ref{fig:9}(c). Most generated data points are not close to real data points, but some points are very close to each other. That results in a peak at the beginning of generated ICD plot. \ul{Any differences of the histograms of ICD and BCD sets will decrease the LS}. Therefore, LS is affected by the isolated peaks of one distance set. Fig.~\ref{fig:9}(d) shows the combined effect. Generated data points are close to real data points and cluster in a few places. Both BCD and generated ICD peaks can be found at the beginning of plot. For the last Fig.~\ref{fig:9}(e), lack of inheritance means generated data are dissimilar from real data. The two kinds of data are distributed separately so that distributions of the three sets are all different, contrary to Fig.~\ref{fig:9}(a); that leads to the lowest LS. 

\subsection{Time Complexity}

\begin{figure}[t]
    \includegraphics[width=.48\textwidth]{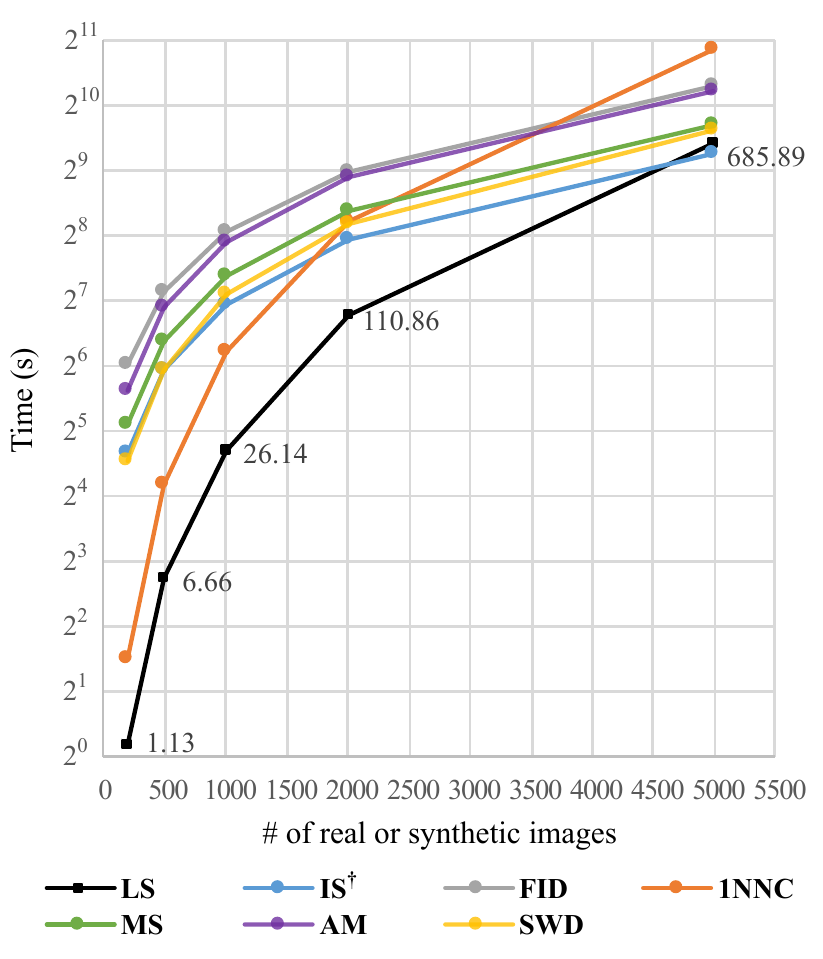}
    \caption{Time cost of measures running on a single core of CPU (i7-6900K). To test time costs, we used same amount of real and generated images (200, 500, 1000, 2000, and 5000) from CIFAR-10 dataset and DCGAN trained on CIFAR-10. $^\dagger$IS only used the generated images.}
    \label{fig:10}
\end{figure}

Both LS and 1NNC use the direct image comparison which is the Euclidean ($l^2$-norm) distance between two images. The main time cost of LS is to calculate ICD and BCD sets. LS’s time complexity for N (Class 1) and M (Class 2) data is about $\mathcal{O}\left(N^2/2+M^2/2+MN\right)$ (two ICD sets and one BCD set). Although 1NNC also uses Euclidean distance between two images, its time complexity is about $\mathcal{O}\left((M+N)^2\right)$, which is double the cost of LS, because it uses the \textit{Leave-One-Out Cross-Validation} for \textit{1-Nearest Neighbor} classifier. For each sample from the $(M+N)$ images, $(M+N-1)$ distances should be calculated to find its nearest neighbor.

The IS, FID, MS, and AM use the \textit{Inception} neural network to process images so that their time costs are greater than LS if running on CPU (i7-6900K). Although running on GPU could accelerate the processing of neural networks, for fair comparisons of time costs, all measures were run on a single core of CPU because 1NNC and LS do not run on GPU currently; but in the future, they also could be accelerated by moving to run on GPU. 

The Fig.~\ref{fig:10} shows for as many as 5,000 samples, LS has uniformly superior performance in terms of time complexity. Although the growth trend shows other measures (except 1NNC) will be running faster than LS at some larger number of samples, we do not need such a large data set to evaluate GANs. Since GAN measures are stable to the growth of amount of data (as shown by Fig.~\ref{fig:3}), our experiments demonstrate that 2,000 samples are adequate for GAN measures.

\subsection{Comparison Summary}
The compared measures have various drawbacks.
The IS, FID, MS, and AM depends on the Inception network pre-trained by ImageNet. In addition, IS lacks the ability to detect overfitting (creativity) and inheritance and FID depends on the Gaussian distribution assumption of feature vectors from the network. The SWD and 1NNC require that the amount of real data be equal to the amount of generated data. The local conditions of distributions will greatly influence results of 1NNC (\textit{e.g.}, it obtains extreme values like 0 or 1 in Table~\ref{tab:2}) because it only considers the 1-nearest neighbor. That there are several required parameters\footnote{More details are in its source codes: \url{https://github.com/koshian2/swd-pytorch}.} such as \texttt{slice\_size} and \texttt{n\_descriptors} is another disadvantage of SWD; both changes of parameters and the randomness of radial projections will influence its results. 

The proposed LS is designed to avoid those disadvantages. We have created three criteria (creativity, diversity, and inheritance) to describe ideal GANs. And we have shown that LS evaluates a GAN by examining the three aspects in a uniform framework. In addition, LS does not need a pre-trained classifier, image analysis methods, nor \textit{a priori} knowledge of distributions. Ranging between 0 and 1 is another merit of LS because we could know how close the performance of a GAN model is to the ideal situation.

We found that the idea of GS~\cite{khrulkov2018geometry} has some similar points to our LS. The GS compares the complexities of the manifold structures, which are built by pairwise distances of samples, between real and generated data. And we think the complexity of data manifold may have some connections to data separability. However, we found the results of GS is too unstable to use. For example, we have computed GS measure twice on 2,000 generated and 2,000 real images from DCGAN and CIFAR-10 (the same test in Section~\ref{cifar-10}); one result is 0.0078 and another is 0.0142 -- it is almost doubled. As Fig.~\ref{fig:3} shown, GS results not only differ on each computation time but also on the amount of samples.

\subsection{Contributions and Future Works}
LS uses a very simple process -- it calculates only Euclidean distances of data and the KS distances between distributions of data distances; those methods are independent of image types, amounts, and sizes. LS offers a distinctly new way to measure the separability of real and generated data. By experiments, it has been verified to be an effective GAN evaluation method by examining the three aspects (creativity, diversity, and inheritance) of ideal GANs. In particular, \ul{LS can provide evidences of the three aspects in the histograms of ICD and BCD sets to explain its results} (\textit{e.g.}, Fig.~\ref{fig:9}). In the future, individual measures (scores) for each aspect could be developed by further analysis of the ICD and BCD sets.

Besides evaluation of GANs, LS could measure data complexity/separability as well. According to Theorem \ref{thm:1}, the LS provides \ul{an effective way to verify whether the distributions of two sample sets are identical for any dimensionality}. Thus, our proposed novel model-independent measure for GAN evaluation has clear advantages in theory and has been demonstrated to be worthwhile for future GAN studies.

Results also show that a GAN that performs well with one type of image may not do so with other types. For example, in Table~\ref{tab:2} and Fig.~\ref{fig:5}, we see that the SNGAN performs much better on \texttt{Hole} images than on other image types. Hence, in future work, we will examine the proposed measure on more types of images and GAN models.

\section{Conclusion}
The novel GAN measure -- LS -- we propose here can directly analyze the generated images without using a pre-trained classifier and it is stable with respect to the amount of images. The strength of LS is that it avoids the disadvantages of compared methods, such as IS and FID, and has fewer constraints and wider applications. Furthermore, LS could evaluate the performance of GANs well, and particularly, provides explanation of results in the three main respects of optimal GANs according to our expectations of ideal generated images. Such explanations help us to deepen our understanding of GANs and of other GAN measures that will help to improve GAN performance.

\appendix
\section{Proof of Theorem \ref{thm:1}}
\label{sec:proof}

Consider two classes $X$ and $Y$ that have the same distribution (distributions have the same shape, position, and support, \textit{i.e.}, the same probability density function) and have sufficient data points to fill their support domains. Suppose $X$ and $Y$ have $N_x$ and $N_y$ data points, and assume the sampling density ratio is $\frac{N_y}{N_x} =\alpha$. Before providing the proof of Theorem \ref{thm:1}, we firstly prove Lemma \ref{lma:1}, which will be used later.

\begin{lemma} \label{lma:1}
If and only if two classes $X$ and $Y$ have the same distribution covering region $\Omega$ and $\frac{N_y}{N_x} =\alpha$, for any sub-region $\Delta \subseteq \Omega$, with $X$ and $Y$ having $n_{xi},n_{yi}$ points, $\frac{n_{yi}}{n_{xi}} =\alpha$ holds.
\end{lemma}

\begin{proof}
Assume the distributions of $X$ and $Y$ are $f(x)$ and $g(y)$. In the union region of $X$ and $Y$, arbitrarily take one tiny cell (region) $\Delta_i$ with $n_{xi}=\Delta_if(x_i)N_x,\  n_{yi}=\Delta_ig(y_j)N_y;\ x_i=y_j$. Then,
\[
\frac{n_{yi}}{n_{xi}}=\frac{\Delta_ig(x_i)N_y}{\Delta_if(x_i)N_x}=\alpha \frac{g(x_i)}{f(x_i)}
\]
Therefore:
\[
\alpha \frac{g(x_i)}{f(x_i)} =\alpha \Leftrightarrow \frac{g(x_i)}{f(x_i)}=1 \Leftrightarrow \forall x_i:g(x_i)=f(x_i)
\]
\qed
\end{proof}
\textbf{Sufficient condition of Theorem \ref{thm:1}.} \textit{When $|\{d_x\}|,|\{d_y\}|\to \infty$, if the two classes $X$ and $Y$ have the same distribution, the distributions of the ICD and BCD sets are identical.}

\begin{figure}[h]
    \includegraphics[width=.48\textwidth]{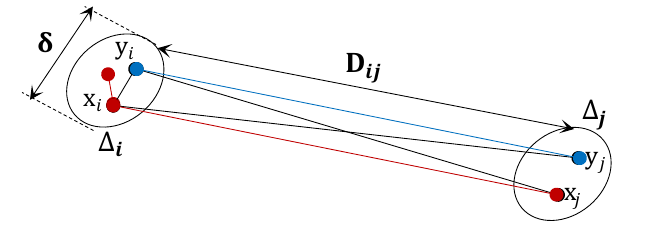}
    \caption{Two non-overlapping small cells}
    \label{fig:a1}
\end{figure}

\begin{proof}
Within the area, select two tiny non-overlapping cells (regions) $\Delta_i$  and $\Delta_j$ (Fig.~\ref{fig:a1}). Since $X$ and $Y$ have the same distribution but in general different densities, the number of points in the two cells $n_{xi},n_{yi};\ n_{xj},n_{yj}$ fulfills:
\[
\frac{n_{yi}}{n_{xi}} =\frac{n_{yj}}{n_{xj}} =\alpha
\]
The scale of cells is $\delta$, the ICDs and BCDs of $X$ and $Y$ data points in cell $\Delta_i$ are approximately $\delta$ because the cell is sufficiently small. By the Definition~\ref{def:1}~and~\ref{def:2}:
\[
d_{x_i}\approx d_{x_i,y_i}\approx \delta;\quad x_i,y_i\in \Delta_i
\]
Similarly, the ICDs and BCDs of $X$ and $Y$ data points between cells $\Delta_i$  and $\Delta_j$ are approximately the distance between the two cells $D_{ij}$:
\[
d_{x_{ij}}\approx d_{x_i,y_j}\approx d_{y_i,x_j}\approx D_{ij};\; x_i,y_i\in \Delta_i;\, x_j,y_j\in \Delta_j
\]
First, divide the whole distribution region into many non-overlapping cells. Arbitrarily select two cells $\Delta_i$  and $\Delta_j$ to examine the ICD set for $X$ and the BCD set for $X$ and $Y$. By Corollaries \ref{cor:1} and \ref{cor:2}: \\
\\
\romannum{1} ) The ICD set for $X$ has two distances: $\delta$ and $D_{ij}$, and their numbers are:
\[
d_{x_i}\approx \delta;\; x_i\in \Delta_i:\; |\{d_{x_i}\}|=\frac{1}{2}n_{xi}(n_{xi}-1)
\]
\[
d_{x_{ij}}\approx D_{ij};\; x_i\in \Delta_i;x_j\in \Delta_j:\; |\{d_{x_{ij}}\}|=n_{xi}n_{xj}
\]
\romannum{2} ) The BCD set for $X$ and $Y$ also has two distances: $\delta$ and $D_{ij}$, and their numbers are:
\[
d_{x_i,y_i}\approx \delta;\; x_i,y_i\in \Delta_i:\; |\{d_{x_i ,y_i}\}|=n_{xi} n_{yi}
\]
\[
d_{x_i,y_j}\approx d_{y_i,x_j}\approx D_{ij};\; x_i,y_i\in \Delta_i;x_j,y_j\in \Delta_j:
\]
\[
|\{d_{x_i,y_j}\}|=n_{xi} n_{yj};\; |\{d_{y_i,x_j}\}|=n_{yi} n_{xj}
\]
Therefore, the proportions of the number of distances with a value of $D_{ij}$ in the ICD and BCD sets are:\\
\\
For ICDs: 
\[
\frac{|\{d_{x_{ij}} \}|}{|\{d_x \}|} =\frac{2n_{xi} n_{xj}}{N_x (N_x-1)}
\]
For BCDs, considering the density ratio:
\[
\frac{|\{d_{x_i,y_j} \}|+|\{d_{y_i,x_j }\}|}{|\{d_{x,y} \}|} =\frac{\alpha n_{xi} n_{xj}+\alpha n_{xi} n_{xj}}{\alpha N_x^2 }=\frac{2n_{xi} n_{xj}}{N_x^2}
\]
The ratio of proportions of the number of distances with a value of $D_{ij}$ in the two sets is:
\[
\frac{N_x (N_x-1)}{N_x^2}=1-\frac{1}{N_x} \to 1 \; \; (N_x\to \infty)
\]
This means that the number of proportions of the number of distances with a value of $D_{ij}$ in the two sets is equal. We then examine the proportions of the number of distances with a value of $\delta$ in the ICD and BCD sets. \\
\\
For ICDs:
\begin{multline*}
\sum_{i} \frac{|\{d_{x_i}\}|}{|\{d_x\}|} = \frac{\sum_{i} [n_{xi} (n_{xi}-1)]}{N_x (N_x-1)} \\ = \frac{\sum_{i} (n_{xi}^2-n_{xi} )}{N_x^2-N_x} = \frac{\sum_{i} (n_xi^2 ) -N_x}{N_x^2-N_x}
\end{multline*}
For BCDs, considering the density ratio: 
\[
\sum_{i} \frac{|\{d_{x_i,y_i } \}|}{|\{d_{x,y})\}|} = \frac{\sum_{i} (n_{xi}^2 )}{N_x^2}
\]
The ratio of proportions of the number of distances with a value of $\delta$ in the two sets is:
\begin{multline*}
\frac{\sum_{i} (n_{xi}^2 ) }{N_x^2 }\cdot \frac{N_x^2-N_x}{\sum_{i} (n_{xi}^2 ) -N_x } \\
=\sum_{i} \left(\frac{n_{xi}^2}{N_x^2} \right) \cdot \frac{1-\frac{1}{N_x}}{\sum_{i} \left(\frac{n_{xi}^2}{N_x^2} \right) -\frac{1}{N_x}}\to 1 \; \; (N_x\to \infty)
\end{multline*}
This means that the number of proportions of the number of distances with a value of $\delta$ in the two sets is equal. \\
\\
In summary, the fact that the proportion of any distance value ($\delta$ or $D_{ij}$) in the ICD set for $X$ and in the BCD set for $X$ and $Y$ is equal indicates that the distributions of the ICD and BCD sets are identical, and a corresponding proof applies to the ICD set for $Y$.
\qed
\end{proof}
\textbf{Necessary condition of Theorem \ref{thm:1}.} \textit{When $|\{d_x\}|,|\{d_y\}|\to \infty$, if the distributions of the ICD and BCD sets are identical, then the two classes $X$ and $Y$ must have the same distribution.}
\begin{remark}
We prove its \textbf{contrapositive}: if $X$ and $Y$ do not have the same distribution, the distributions of the ICD and BCD sets are not identical. We then apply proof by \textbf{contradiction}: suppose that $X$ and $Y$ do not have the same distribution, but the distributions of the ICD and BCD sets are identical.
\end{remark}

\begin{proof}
Suppose classes $X$ and $Y$ have the data points $N_x, N_y$, which $\frac{N_y}{N_x} =\alpha $. Divide their distribution area into many non-overlapping tiny cells (regions). In the $i$-th cell $\Delta_i$, since distributions of $X$ and $Y$ are different, according to Lemma \ref{lma:1}, the number of points in the cell $n_{xi},n_{yi}$ fulfills:
\[
\frac{n_{yi}}{n_{xi}} = \alpha _i; \; \; \exists \alpha _i \neq \alpha
\]
The scale of cells is $\delta$ and the ICDs and BCDs of the $X$ and $Y$ points in cell $\Delta_i$ are approximately $\delta$ because the cell is sufficiently small.
\[
d_{x_i}\approx d_{y_i}\approx d_{x_i,y_i}\approx \delta; \; \; x_i,y_i\in \Delta_i
\]
In the $i$-th cell $\Delta_i$:\\
\\
\romannum{1}) The ICD of $X$ is $\delta$, with a proportion of:
\begin{multline} \label{eq:1}
    \sum_{i} \frac{|\{d_{x_i}\}|}{|\{d_x\}|} = \frac{\sum_{i} [n_{xi} (n_{xi}-1)]}{N_x (N_x-1)} \\
    =\frac {\sum_{i} (n_{xi}^2-n_{xi} )}{N_x^2-N_x}=\frac{\sum_{i} (n_{xi}^2 ) -N_x}{N_x^2-N_x}
\end{multline}
\romannum{2}) The ICD of $Y$ is $\delta$, with a proportion of:
\begin{multline} \label{eq:2}
    \sum_{i} \frac{|\{d_{y_i}\}|}{|\{d_y\}|} = \frac{\sum_{i} [n_{yi} (n_{yi}-1)]}{N_y (N_y-1)}=\frac {\sum_{i} (n_{yi}^2-n_{yi} )}{N_y^2-N_y}\\
    =\frac{\sum_{i} (n_{yi}^2 ) -N_y}{N_y^2-N_y}\Bigg\rvert_{\substack{N_y=\alpha N_x \\ n_{yi} = \alpha _i n_{xi}}} = \frac{\sum_{i} (\alpha _i^2 n_{xi}^2 ) -\alpha N_x}{\alpha^2 N_x^2-\alpha N_x}
\end{multline}
\romannum{3}) The BCD of $X$ and $Y$ is $\delta$, with a proportion of:
\begin{multline} \label{eq:3}
    \sum_{i} \frac{|\{d_{x_i,y_i} \}|}{|\{d_{x,y} \}|}=\frac {\sum_{i} (n_{xi} n_{yi} ) }{N_x N_y}=\frac {\sum_{i} (\alpha _i n_{xi}^2 ) }{\alpha N_x^2}
\end{multline}
For the distributions of the two sets to be identical, the ratio of proportions of the number of distances with a value of $\delta$ in the two sets must be 1, that is $\frac{(\ref{eq:3})}{(\ref{eq:1})}=\frac{(\ref{eq:3})}{(\ref{eq:2})}=1$. Therefore:

\begin{multline} \label{eq:4}
    \frac{(\ref{eq:3})}{(\ref{eq:1})}= \frac {\sum_{i} (\alpha _i n_{xi}^2 ) }{\alpha N_x^2} \cdot \frac{N_x^2-N_x}{\sum_{i} (n_{xi}^2 )-N_x}\\
    = \frac{1}{\alpha N_x^2}\sum_{i} (\alpha _i n_{xi}^2 )\cdot \frac {1-\frac {1}{N_x} }{\frac {1}{N_x^2} \sum_{i}(n_{xi}^2 ) -\frac {1}{N_x}}\Bigg\rvert_{N_x\to \infty} \\
    =\frac {1}{\alpha}\cdot \frac{\sum_{i} (\alpha_i n_{xi}^2 ) }{\sum_{i}(n_{xi}^2 )}=1
\end{multline}
Similarly,
\begin{multline} \label{eq:5}
    \frac{(\ref{eq:3})}{(\ref{eq:2})}= \frac {\sum_{i} (\alpha _i n_{xi}^2 ) }{\alpha N_x^2} \cdot \frac{\alpha^2 N_x^2-\alpha N_x}{\sum_{i} (\alpha _i^2 n_{xi}^2 ) -\alpha N_x}\\
    = \frac{\sum_{i} (\alpha _i n_{xi}^2 )}{N_x^2}\cdot \frac {\alpha-\frac {1}{N_x} }{\frac {1}{N_x^2} \sum_{i} (\alpha _i^2 n_{xi}^2 ) -\frac {\alpha}{N_x}}\Bigg\rvert_{N_x\to \infty} \\
    =\alpha \cdot \frac{\sum_{i} (\alpha_i n_{xi}^2 ) }{\sum_{i} (\alpha _i^2 n_{xi}^2 )}=1
\end{multline}
To eliminate the $\sum_{i} (\alpha _i n_{xi}^2 )$ by considering the Eq.~\ref{eq:4}~and~\ref{eq:5}, we have:
\[
\sum_{i}(n_{xi}^2 )=\frac{\sum_{i} (\alpha _i^2 n_{xi}^2 )}{\alpha^2}
\]
Let $\rho_i=\left(\frac{\alpha_i}{\alpha}\right)^2$, then,
\[
\sum_{i}(n_{xi}^2 )=\sum_{i} (\rho_i n_{xi}^2 )
\]
Since $n_{xi}$ could be any value, to hold the equation requires $\rho_i=1$. Hence:
\[
\forall \rho_i=\left(\frac{\alpha_i}{\alpha}\right)^2=1 \Rightarrow\forall \alpha_i=\alpha
\]
This contradicts $\exists \alpha_i\neq \alpha$. Therefore, the contrapositive proposition has been proved. 
\qed
\end{proof}

\section{DSI for Multi-class Dataset}
\label{sec:multiclass}
In general, for a $n$-class dataset, the process to obtain its DSI is:

\begin{enumerate}
    \item Compute $n$ ICD sets for each class: $\{d_{C_i}\};\; i=1,2,\cdots,n$.
    
    \item Compute $n$ BCD sets for each class. For the $i$-th class of data $C_i$, the BCD set is the set of distances between any two points in $C_i$ and $\overline{C_i}$ (other classes, not $C_i$): $\{d_{C_i,\overline{C_i}}\}$.
    
    \item Compute the $n$ KS distances between ICD and BCD sets for each class: $s_i=KS(\{d_{C_i }\},\{d_{C_i,\overline{C_i}}\})$.
    
    \item The final DSI is derived from the $n$ KS distances by requirements. \textit{E.g.}, their average: $DSI(\{C_i \})=\frac{\sum s_i}{n}$ or the maximum value: $DSI(\{C_i \})=\max\{s_i\}$.
\end{enumerate}


%
\section*{Compliance with ethical standards}
\textbf{Conflict of interest} The authors declare that they have no conflict of interest.

\bibliographystyle{spbasic}      

\bibliography{refs}   

\begin{thebibliography}{45}
\providecommand{\natexlab}[1]{#1}
\providecommand{\url}[1]{{#1}}
\providecommand{\urlprefix}{URL }
\expandafter\ifx\csname urlstyle\endcsname\relax
  \providecommand{\doi}[1]{DOI~\discretionary{}{}{}#1}\else
  \providecommand{\doi}{DOI~\discretionary{}{}{}\begingroup
  \urlstyle{rm}\Url}\fi
\providecommand{\eprint}[2][]{\url{#2}}

\bibitem[{noa(2021)}]{noauthor_scipystatswasserstein_distance_nodate}
 (2021) scipy.stats.wasserstein\_distance — {SciPy} v1.6.1 {Reference}
  {Guide}.
  \urlprefix\url{https://docs.scipy.org/doc/scipy/reference/generated/scipy.stats.wasserstein_distance.html}

\bibitem[{Arjovsky et~al.(2017)Arjovsky, Chintala, and
  Bottou}]{pmlr-Wassersteingan}
Arjovsky M, Chintala S, Bottou L (2017) {W}asserstein generative adversarial
  networks. In: Precup D, Teh YW (eds) Proceedings of the 34th International
  Conference on Machine Learning, PMLR, International Convention Centre,
  Sydney, Australia, Proceedings of Machine Learning Research, vol~70, pp
  214--223, \urlprefix\url{http://proceedings.mlr.press/v70/arjovsky17a.html}

\bibitem[{Backes et~al.(2012)Backes, Casanova, and Bruno}]{Backes2012Color}
Backes AR, Casanova D, Bruno OM (2012) Color texture analysis based on fractal
  descriptors. Pattern Recognition 45(5):1984--1992,
  \doi{10.1016/j.patcog.2011.11.009}

\bibitem[{Bonneel et~al.(2015)Bonneel, Rabin, Peyré, and
  Pfister}]{bonneel_sliced_2015}
Bonneel N, Rabin J, Peyré G, Pfister H (2015) Sliced and {Radon} {Wasserstein}
  {Barycenters} of {Measures}. Journal of Mathematical Imaging and Vision
  51(1):22--45, \doi{10.1007/s10851-014-0506-3}

\bibitem[{Borji(2019)}]{Borji2019Pros}
Borji A (2019) Pros and cons of gan evaluation measures. Computer Vision and
  Image Understanding 179:41--65, \doi{10.1016/j.cviu.2018.10.009}

\bibitem[{Che et~al.(2016)Che, Li, Jacob, Bengio, and Li}]{che2016mode}
Che T, Li Y, Jacob A, Bengio Y, Li W (2016) Mode {Regularized} {Generative}
  {Adversarial} {Networks}.
  \urlprefix\url{https://openreview.net/forum?id=HJKkY35le}

\bibitem[{Deng et~al.(2009)Deng, Dong, Socher, Li, Li, and
  Fei-Fei}]{Deng2009ImageNet:}
Deng J, Dong W, Socher R, Li LJ, Li K, Fei-Fei L (2009) Imagenet: A large-scale
  hierarchical image database. 2009 IEEE Conference on Computer Vision and
  Pattern Recognition, pp 248--255, \doi{10.1109/CVPR.2009.5206848}, iSSN:
  1063-6919

\bibitem[{Frank J.~Massey(1951)}]{Kolmogorov1951}
Frank J~Massey J (1951) The {Kolmogorov}-{Smirnov} {Test} for {Goodness} of
  {Fit}. Journal of the American Statistical Association 46(253):68--78,
  \doi{10.1080/01621459.1951.10500769}

\bibitem[{Geirhos et~al.(2019)Geirhos, Rubisch, Michaelis, Bethge, Wichmann,
  and Brendel}]{Geirhos2019ImageNet-trained}
Geirhos R, Rubisch P, Michaelis C, Bethge M, Wichmann FA, Brendel W (2019)
  Imagenet-trained cnns are biased towards texture; increasing shape bias
  improves accuracy and robustness. In: 7th International Conference on
  Learning Representations, {ICLR} 2019, New Orleans, LA, USA, May 6-9, 2019,
  OpenReview.net, \urlprefix\url{https://openreview.net/forum?id=Bygh9j09KX}

\bibitem[{Goodfellow et~al.(2014)Goodfellow, Pouget-Abadie, Mirza, Xu,
  Warde-Farley, Ozair, Courville, and Bengio}]{Goodfellow2014Generative}
Goodfellow I, Pouget-Abadie J, Mirza M, Xu B, Warde-Farley D, Ozair S,
  Courville A, Bengio Y (2014) Generative adversarial nets. In: Ghahramani Z,
  Welling M, Cortes C, Lawrence ND, Weinberger KQ (eds) Advances in Neural
  Information Processing Systems 27, Curran Associates, Inc., p 2672–2680,
  \urlprefix\url{http://papers.nips.cc/paper/5423-generative-adversarial-nets.pdf}

\bibitem[{Gretton et~al.(2012)Gretton, Borgwardt, Rasch, Sch{{\"o}}lkopf, and
  Smola}]{JMLR:v13:gretton12a}
Gretton A, Borgwardt KM, Rasch MJ, Sch{{\"o}}lkopf B, Smola A (2012) A kernel
  two-sample test. Journal of Machine Learning Research 13(25):723--773,
  \urlprefix\url{http://jmlr.org/papers/v13/gretton12a.html}

\bibitem[{Gulrajani et~al.(2017)Gulrajani, Ahmed, Arjovsky, Dumoulin, and
  Courville}]{Gulrajani2017Improved}
Gulrajani I, Ahmed F, Arjovsky M, Dumoulin V, Courville AC (2017) Improved
  training of wasserstein gans. In: Guyon I, Luxburg UV, Bengio S, Wallach H,
  Fergus R, Vishwanathan S, Garnett R (eds) Advances in Neural Information
  Processing Systems 30, Curran Associates, Inc., p 5767–5777,
  \urlprefix\url{http://papers.nips.cc/paper/7159-improved-training-of-wasserstein-gans.pdf}

\bibitem[{Heusel et~al.(2017)Heusel, Ramsauer, Unterthiner, Nessler, and
  Hochreiter}]{Heusel2017GANs}
Heusel M, Ramsauer H, Unterthiner T, Nessler B, Hochreiter S (2017) Gans
  trained by a two time-scale update rule converge to a local nash equilibrium.
  In: Guyon I, Luxburg UV, Bengio S, Wallach H, Fergus R, Vishwanathan S,
  Garnett R (eds) Advances in Neural Information Processing Systems 30, Curran
  Associates, Inc., p 6626–6637

\bibitem[{Hindupur(2018)}]{Hindupur2018the-gan-zoo:}
Hindupur A (2018) the-gan-zoo: A list of all named GANs!
  \urlprefix\url{https://github.com/hindupuravinash/the-gan-zoo},
  original-date: 2017-04-14T16:45:24Z

\bibitem[{Hong et~al.(2019)Hong, Hwang, Yoo, and Yoon}]{Hong2019How}
Hong Y, Hwang U, Yoo J, Yoon S (2019) How generative adversarial networks and
  their variants work: An overview. ACM Computing Surveys 52(1):1--43,
  \doi{10.1145/3301282}

\bibitem[{Im et~al.(2016)Im, Kim, Jiang, and Memisevic}]{im2016generating}
Im DJ, Kim CD, Jiang H, Memisevic R (2016) Generating images with recurrent
  adversarial networks. arXiv preprint arXiv:160205110

\bibitem[{Isola et~al.(2017)Isola, Zhu, Zhou, and Efros}]{isola2017image}
Isola P, Zhu JY, Zhou T, Efros AA (2017) Image-to-image translation with
  conditional adversarial networks. In: Proceedings of the IEEE conference on
  computer vision and pattern recognition, pp 1125--1134

\bibitem[{Khrulkov and Oseledets(2018)}]{khrulkov2018geometry}
Khrulkov V, Oseledets I (2018) Geometry score: A method for comparing
  generative adversarial networks. In: International Conference on Machine
  Learning, PMLR, pp 2621--2629

\bibitem[{Kullback and Leibler(1951)}]{Kullback1951On}
Kullback S, Leibler RA (1951) On information and sufficiency. The Annals of
  Mathematical Statistics 22(1):79--86

\bibitem[{LeCun et~al.(2010)LeCun, Cortes, and Burges}]{LeCun2010MNIST}
LeCun Y, Cortes C, Burges CJ (2010) Mnist handwritten digit database

\bibitem[{Ledig et~al.(2017)Ledig, Theis, Huszar, Caballero, Cunningham,
  Acosta, Aitken, Tejani, Totz, Wang, and Shi}]{Ledig2017Photo-Realistic}
Ledig C, Theis L, Huszar F, Caballero J, Cunningham A, Acosta A, Aitken A,
  Tejani A, Totz J, Wang Z, Shi W (2017) Photo-realistic single image
  super-resolution using a generative adversarial network. 2017 IEEE Conference
  on Computer Vision and Pattern Recognition (CVPR), IEEE, Honolulu, HI, pp
  105--114, \doi{10.1109/CVPR.2017.19},
  \urlprefix\url{http://ieeexplore.ieee.org/document/8099502/}

\bibitem[{Lehmann and Romano(2006)}]{lehmann2006testing}
Lehmann EL, Romano JP (2006) Testing statistical hypotheses. Springer Science
  \& Business Media

\bibitem[{Li et~al.(2017)Li, Liang, Wei, Xu, Feng, and Yan}]{Li2017Perceptual}
Li J, Liang X, Wei Y, Xu T, Feng J, Yan S (2017) Perceptual generative
  adversarial networks for small object detection. 2017 IEEE Conference on
  Computer Vision and Pattern Recognition (CVPR), IEEE, Honolulu, HI, pp
  1951--1959, \doi{10.1109/CVPR.2017.211},
  \urlprefix\url{http://ieeexplore.ieee.org/document/8099694/}

\bibitem[{Lopez-Paz and Oquab(2017)}]{Lopez-Paz2017Revisiting}
Lopez-Paz D, Oquab M (2017) Revisiting classifier two-sample tests.
  \urlprefix\url{https://openreview.net/forum?id=SJkXfE5xx}

\bibitem[{Maaten and Hinton(2008)}]{Maaten2008Visualizing}
Maaten vdL, Hinton G (2008) Visualizing data using t-sne. Journal of Machine
  Learning Research 9(Nov):2579--2605

\bibitem[{Mao et~al.(2017)Mao, Li, Xie, Lau, Wang, and Smolley}]{Mao2017Least}
Mao X, Li Q, Xie H, Lau RY, Wang Z, Smolley SP (2017) Least squares generative
  adversarial networks. 2017 IEEE International Conference on Computer Vision
  (ICCV), pp 2813--2821, \doi{10.1109/ICCV.2017.304}, iSSN: 2380-7504

\bibitem[{Miyato et~al.(2018)Miyato, Kataoka, Koyama, and
  Yoshida}]{Miyato2018Spectral}
Miyato T, Kataoka T, Koyama M, Yoshida Y (2018) Spectral normalization for
  generative adversarial networks.
  \urlprefix\url{https://openreview.net/forum?id=B1QRgziT-}

\bibitem[{Pan et~al.(2019)Pan, Yu, Yi, Khan, Yuan, and Zheng}]{Pan2019Recent}
Pan Z, Yu W, Yi X, Khan A, Yuan F, Zheng Y (2019) Recent progress on generative
  adversarial networks (gans): A survey. IEEE Access 7:36322--36333,
  \doi{10.1109/ACCESS.2019.2905015}

\bibitem[{Radford et~al.(2016)Radford, Metz, and
  Chintala}]{Radford2016Unsupervised}
Radford A, Metz L, Chintala S (2016) Unsupervised representation learning with
  deep convolutional generative adversarial networks.
  \urlprefix\url{http://arxiv.org/abs/1511.06434}

\bibitem[{Ramdas et~al.(2017)Ramdas, Trillos, and
  Cuturi}]{Ramdas2017Wasserstein}
Ramdas A, Trillos NG, Cuturi M (2017) On wasserstein two-sample testing and
  related families of nonparametric tests. Entropy 19(2),
  \doi{10.3390/e19020047}

\bibitem[{Rüschendorf(1985)}]{ruschendorf_wasserstein_1985}
Rüschendorf L (1985) The {Wasserstein} distance and approximation theorems.
  Probability Theory and Related Fields 70(1):117--129,
  \doi{10.1007/BF00532240},
  \urlprefix\url{https://link.springer.com/article/10.1007/BF00532240}

\bibitem[{Salimans et~al.(2016)Salimans, Goodfellow, Zaremba, Cheung, Radford,
  Chen, and Chen}]{Salimans2016Improved}
Salimans T, Goodfellow I, Zaremba W, Cheung V, Radford A, Chen X, Chen X (2016)
  Improved techniques for training gans. In: Lee DD, Sugiyama M, Luxburg UV,
  Guyon I, Garnett R (eds) Advances in Neural Information Processing Systems
  29, Curran Associates, Inc., p 2234–2242,
  \urlprefix\url{http://papers.nips.cc/paper/6125-improved-techniques-for-training-gans.pdf}

\bibitem[{Santurkar et~al.(2018)Santurkar, Schmidt, and
  Madry}]{santurkar2018classification}
Santurkar S, Schmidt L, Madry A (2018) A classification-based study of
  covariate shift in gan distributions. In: International Conference on Machine
  Learning, PMLR, pp 4480--4489

\bibitem[{Snell et~al.(2017)Snell, Ridgeway, Liao, Roads, Mozer, and
  Zemel}]{snell2017learning}
Snell J, Ridgeway K, Liao R, Roads BD, Mozer MC, Zemel RS (2017) Learning to
  generate images with perceptual similarity metrics. In: 2017 IEEE
  International Conference on Image Processing (ICIP), IEEE, pp 4277--4281

\bibitem[{Szegedy et~al.(2016)Szegedy, Vanhoucke, Ioffe, Shlens, and
  Wojna}]{Szegedy2016Rethinking}
Szegedy C, Vanhoucke V, Ioffe S, Shlens J, Wojna Z (2016) Rethinking the
  inception architecture for computer vision. 2016 IEEE Conference on Computer
  Vision and Pattern Recognition (CVPR), IEEE, Las Vegas, NV, USA, pp
  2818--2826, \doi{10.1109/CVPR.2016.308},
  \urlprefix\url{http://ieeexplore.ieee.org/document/7780677/}

\bibitem[{Theis et~al.(2016)Theis, van~den Oord, and Bethge}]{Theis2016note}
Theis L, van~den Oord A, Bethge M (2016) A note on the evaluation of generative
  models. In: Bengio Y, LeCun Y (eds) 4th International Conference on Learning
  Representations, {ICLR} 2016, San Juan, Puerto Rico, May 2-4, 2016,
  Conference Track Proceedings, \urlprefix\url{http://arxiv.org/abs/1511.01844}

\bibitem[{Tolstikhin et~al.(2017)Tolstikhin, Gelly, Bousquet, Simon-Gabriel,
  and Sch{\"o}lkopf}]{tolstikhin2017adagan}
Tolstikhin IO, Gelly S, Bousquet O, Simon-Gabriel CJ, Sch{\"o}lkopf B (2017)
  Adagan: Boosting generative models. In: NIPS

\bibitem[{Wang et~al.(2018)Wang, Xu, Wang, and Tao}]{Wang2018Perceptual}
Wang C, Xu C, Wang C, Tao D (2018) Perceptual adversarial networks for
  image-to-image transformation. IEEE Transactions on Image Processing
  27(8):4066--4079, \doi{10.1109/TIP.2018.2836316}

\bibitem[{Wu et~al.(2019)Wu, Zheng, Zhang, and Huang}]{Wu2019GP-GAN:}
Wu H, Zheng S, Zhang J, Huang K (2019) Gp-gan: Towards realistic
  high-resolution image blending. the 27th ACM International Conference, ACM
  Press, Nice, France, pp 2487--2495, \doi{10.1145/3343031.3350944},
  \urlprefix\url{http://dl.acm.org/citation.cfm?doid=3343031.3350944}

\bibitem[{Yang et~al.(2017)Yang, Kannan, Batra, and Parikh}]{YangKBP17}
Yang J, Kannan A, Batra D, Parikh D (2017) {LR-GAN:} layered recursive
  generative adversarial networks for image generation. In: 5th International
  Conference on Learning Representations, {ICLR} 2017, Toulon, France, April
  24-26, 2017, Conference Track Proceedings, OpenReview.net,
  \urlprefix\url{https://openreview.net/forum?id=HJ1kmv9xx}

\bibitem[{Yi et~al.(2017)Yi, Zhang, Tan, and Gong}]{Yi2017DualGAN:}
Yi Z, Zhang H, Tan P, Gong M (2017) Dualgan: Unsupervised dual learning for
  image-to-image translation. 2017 IEEE International Conference on Computer
  Vision (ICCV), IEEE, Venice, pp 2868--2876, \doi{10.1109/ICCV.2017.310},
  \urlprefix\url{http://ieeexplore.ieee.org/document/8237572/}

\bibitem[{Zeng et~al.(2017)Zeng, Lu, and Borji}]{zeng2017statistics}
Zeng Y, Lu H, Borji A (2017) Statistics of deep generated images. arXiv
  preprint arXiv:170802688

\bibitem[{Zhang et~al.(2019)Zhang, Goodfellow, Metaxas, and
  Odena}]{Zhang2019Self-Attention}
Zhang H, Goodfellow I, Metaxas D, Odena A (2019) Self-attention generative
  adversarial networks. International Conference on Machine Learning, pp
  7354--7363, \urlprefix\url{http://proceedings.mlr.press/v97/zhang19d.html},
  iSSN: 1938-7228 section: Machine Learning

\bibitem[{Zhang et~al.(2018)Zhang, Song, and Qi}]{zhang2018decoupled}
Zhang Z, Song Y, Qi H (2018) Decoupled learning for conditional adversarial
  networks. In: 2018 IEEE Winter Conference on Applications of Computer Vision
  (WACV), IEEE, pp 700--708

\bibitem[{Zhou et~al.(2018)Zhou, Cai, Rong, Song, Ren, Zhang, Wang, and
  Yu}]{zhou2018activation}
Zhou Z, Cai H, Rong S, Song Y, Ren K, Zhang W, Wang J, Yu Y (2018) Activation
  maximization generative adversarial nets. In: International Conference on
  Learning Representations,
  \urlprefix\url{https://openreview.net/forum?id=HyyP33gAZ}

\end{thebibliography}


\end{document}